\documentclass{article}

\usepackage{microtype}
\usepackage{graphicx}
\usepackage{booktabs} 
\usepackage{amsfonts}       
\usepackage{nicefrac}       
\usepackage{microtype}      
\usepackage{mathtools}
\usepackage{amsthm}
\usepackage{amssymb}
\usepackage{amsmath}
\usepackage{amsfonts}
\usepackage{booktabs}
\usepackage{siunitx}
\usepackage{arydshln}
\usepackage{subcaption}
\usepackage{physics}
\usepackage{soul}

\usepackage{epstopdf}

\usepackage{wrapfig}
\usepackage{tabularx}
\usepackage{floatrow}
\usepackage{arydshln}
\usepackage{array}

\usepackage{pifont}
\newcommand{\cmark}{\ding{51}}%
\newcommand{\xmark}{\ding{55}}%

\newcolumntype{L}[1]{>{\raggedright\arraybackslash}m{#1}} 
\newcolumntype{C}[1]{>{\centering\arraybackslash}m{#1}} 
\newcolumntype{R}[1]{>{\raggedleft\arraybackslash}m{#1}} 
\newcolumntype{N}{@{}m{0pt}@{}}

\newtheorem{theorem}{Theorem}
\newtheorem{lemma}[theorem]{Lemma}

\newtheorem{remark}[theorem]{Remark}

\newcommand{\Lip}{\mathrm{Lip}}

\newcommand{\ricky}[1]{\textcolor{red}{((RC: #1))}}

\usepackage{hyperref}
\usepackage{xcolor}

\usepackage[accepted]{icml2019}

\icmltitlerunning{Invertible Residual Networks}
\begin{document}

\twocolumn[
\icmltitle{Invertible Residual Networks}



\icmlsetsymbol{equal}{*}

\begin{icmlauthorlist}
\icmlauthor{Jens Behrmann}{equal,br,to}
\icmlauthor{Will Grathwohl}{equal,to}
\icmlauthor{Ricky T. Q. Chen}{to}
\icmlauthor{David Duvenaud}{to}
\icmlauthor{J\"orn-Henrik Jacobsen}{equal,to}
\end{icmlauthorlist}

\icmlaffiliation{br}{University of Bremen, Center for Industrial Mathematics}
\icmlaffiliation{to}{Vector Institute and University of Toronto}

\icmlcorrespondingauthor{Jens Behrmann}{jensb@uni-bremen.de}
\icmlcorrespondingauthor{J\"orn-Henrik Jacobsen}{j.jacobsen@vectorinstitute.ai}

\icmlkeywords{Machine Learning, ICML}

\vskip 0.3in
]



\printAffiliationsAndNotice{\icmlEqualContribution} 

\begin{abstract}

We show that standard ResNet architectures can be made invertible, allowing the same model to be used for classification, density estimation, and generation. Typically, enforcing invertibility requires partitioning dimensions or restricting network architectures. In contrast, our approach only requires adding a simple normalization step during training, already available in standard frameworks. Invertible ResNets define a generative model which can be trained by maximum likelihood on unlabeled data. To compute likelihoods, we introduce a tractable approximation to the Jacobian log-determinant of a residual block. Our empirical evaluation shows that invertible ResNets perform competitively with both state-of-the-art image classifiers and flow-based generative models, something that has not been previously achieved with a single architecture.
\end{abstract}

\section{Introduction}

One of the main appeals of neural network-based models is that a single model architecture can often be used to solve a variety of related tasks. However, many recent advances are based on special-purpose solutions tailored to particular domains. State-of-the-art architectures in unsupervised learning, for instance, are becoming increasingly domain-specific~\cite{oord2016pixel,kingma2018glow,DBLP:conf/icml/ParmarVUKSKT18,karras2018style,van2016wavenet}.
On the other hand, one of the most successful feed-forward architectures for discriminative learning are deep residual networks~\cite{he2016deep,zagoruyko2016wide}, which differ considerably from their generative counterparts.
This divide makes it complicated to choose or design a suitable architecture for a given task. It also makes it hard for discriminative tasks to benefit from unsupervised learning.
We bridge this gap with a new class of architectures that perform well in both domains.



\begin{figure}
    \centering
    {{\includegraphics[width=.95\linewidth]{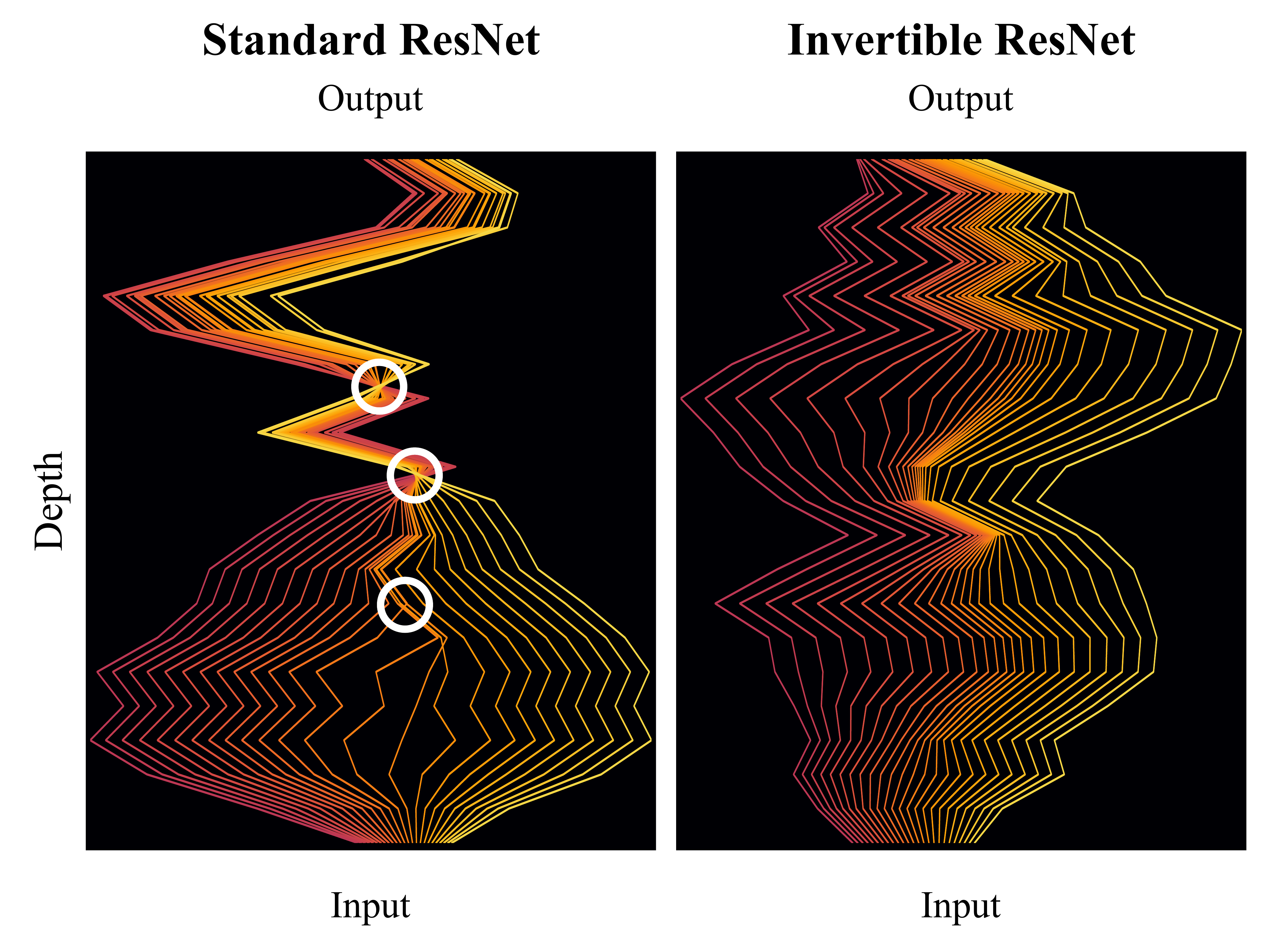} }}
    \label{fig:imgen_cropped}
    \vspace{-5mm}
     \caption{Dynamics of a standard residual network (left) and invertible residual network (right). Both networks map the interval $[-2, 2]$ to: 1) noisy $x^3$-function at half depth and 2) noisy identity function at full depth. Invertible ResNets describe a bijective continuous dynamics while regular ResNets result in crossing and collapsing paths (circled in white) which correspond to non-bijective continuous dynamics. Due to collapsing paths, standard ResNets are not a valid density model.}
     \label{fig:crossingPaths}
\end{figure}

To achieve this, we focus on reversible networks which have been shown to produce competitive performance on discriminative~\cite{gomez2017reversible,jacobsen2018irevnet} and generative~\cite{dinh2014nice,dinh2016density,kingma2018glow} tasks independently, albeit in the same model paradigm. They typically rely on fixed dimension splitting heuristics, but common splittings interleaved with non-volume conserving elements are constraining and their choice has a significant impact on performance~\cite{kingma2018glow,dinh2016density}. This makes building reversible networks a difficult task. In this work we show that these exotic designs, necessary for competitive density estimation performance, can severely hurt discriminative performance. 

To overcome this problem, we leverage the viewpoint of ResNets as an Euler discretization of ODEs \citep{haberRuthotto, ruthottoHaber, luODE, naisnet} and prove that invertible ResNets (i-ResNets) can be constructed by simply changing the normalization scheme of standard ResNets. 

As an intuition, Figure \ref{fig:crossingPaths} visualizes the differences in the dynamics learned by standard and invertible ResNets. 

This approach allows unconstrained architectures for each residual block, while only requiring a Lipschitz constant smaller than one for each block. We demonstrate that this restriction negligibly impacts performance when building image classifiers - they perform on par with their non-invertible counterparts on classifying MNIST, CIFAR10 and CIFAR100 images. 

We then show how i-ResNets can be trained as maximum likelihood generative models on unlabeled data. To compute likelihoods, we introduce a tractable approximation to the Jacobian determinant of a residual block. Like FFJORD~\cite{ffjord}, i-ResNet flows have unconstrained (free-form) Jacobians, allowing them to learn more expressive transformations than the triangular mappings used in other reversible models. Our empirical evaluation shows that i-ResNets perform competitively with both state-of-the-art image classifiers and flow-based generative models, bringing general-purpose architectures one step closer to reality.\footnote{Official code release: \url{https://github.com/jhjacobsen/invertible-resnet}}

\section{Enforcing Invertibility in ResNets}
There is a remarkable similarity between ResNet architectures and Euler's method for ODE initial value problems:
\begin{align*}
    x_{t+1} & \leftarrow x_t + g_{\theta_t}(x_t) \\
    x_{t+1} & \leftarrow x_t + h f_{\theta_t}(x_t)
\end{align*}
where $x_t \in \mathbb{R}^d$ represent activations or states, $t$ represents layer indices or time, $h>0$ is a step size, and $g_{\theta_t}$ is a residual block.
This connection has attracted research at the intersection of deep learning and dynamical systems \citep{luODE, haberRuthotto, ruthottoHaber, chen2018neural}.
However, little attention has been paid to the dynamics backwards in time 
\begin{align*}
    x_t & \leftarrow x_{t+1} - g_{\theta_t}(x_t) \\
    x_t & \leftarrow x_{t+1} - h f_{\theta_t}(x_t)
\end{align*}
which amounts to the implicit backward Euler discretization. In particular, solving the dynamics backwards in time would implement an inverse of the corresponding ResNet.
The following theorem states that a simple condition suffices to make the dynamics solvable and thus renders the ResNet invertible:
\begin{theorem}[Sufficient condition for invertible ResNets]
\label{thm:invertibility}
Let $F_\theta: \mathbb{R}^d \rightarrow \mathbb{R}^d$ with $F_\theta = (F^1_{\theta} \circ \ldots \circ F^T_{\theta})$ denote a ResNet with blocks $F^t_{\theta} = I + g_{\theta_t}$. Then, the ResNet $F_\theta$ is invertible if
\begin{align*}
    \Lip(g_{\theta_t}) < 1, \text{ for all } t=1, \ldots, T ,
\end{align*}
where $\Lip(g_{\theta_t})$ is the Lipschitz-constant of $g_{\theta_t}$.
\end{theorem}
Note that this condition is not necessary for invertibility.
Other approaches \citep{dinh2014nice, dinh2016density, jacobsen2018irevnet, Chang2018ReversibleAF, kingma2018glow} rely on partitioning dimensions or autoregressive structures to create analytical inverses. 

While enforcing $\Lip(g) < 1$ makes the ResNet invertible, we have no analytic form of this inverse. However, we can obtain it through a simple fixed-point iteration, see Algorithm \ref{alg:inverse}. Note, that the starting value for the fixed-point iteration can be any vector, because the fixed-point is unique. However, using the output $y = x + g(x)$ as the initialization $x^0 := y$  is a good starting point since $y$ was obtained from $x$ only via a bounded perturbation of the identity. From the Banach fixed-point theorem we have
\begin{align}
\label{eq:fixedpoint}
    \|x - x^n\|_2 \leq \frac{\Lip(g)^n}{1-\Lip(g)} \|x^1 - x^0\|_2.
\end{align}
Thus, the convergence rate is exponential in the number of iterations $n$ and smaller Lipschitz constants will yield faster convergence.
\begin{algorithm}[t]
   \caption{Inverse of i-ResNet layer via fixed-point iteration.}
   \label{alg:inverse}
\begin{algorithmic}
   \STATE {\bfseries Input:} output from residual layer $y$, contractive residual block $g$, number of fixed-point iterations $n$
   \STATE Init: $x^0 := y$
   \FOR{$i=0, \ldots, n$}
   \STATE $x^{i+1} := y - g(x^i)$
   \ENDFOR
\end{algorithmic}
\end{algorithm}

Additional to invertibility, a contractive residual block also renders the residual layer bi-Lipschitz. 
\begin{lemma}[Lipschitz constants of Forward and Inverse]
\label{thm:lipForwardInverse}
Let $F(x) = x + g(x)$ with $\Lip(g) = L < 1$ denote the residual layer. Then, it holds 
\begin{align*}
    \Lip(F) \leq 1 + L \quad \text{ and } \quad \Lip(F^{-1}) \leq \frac{1}{1-L}.
\end{align*}
\end{lemma}
Hence by design, invertible ResNets offer stability guarantees for both their forward and inverse mapping. In the following section, we discuss approaches to enforce the Lipschitz condition.

\subsection{Satisfying the Lipschitz Constraint}

We implement residual blocks as a composition of contractive nonlinearities $\phi$ (e.g. ReLU, ELU, tanh) and linear mappings. 

For example, in our convolutional networks $g = W_3 \phi(W_2 \phi(W_1))$, where $W_i$ are convolutional layers. Hence,
\begin{align*}
    \Lip (g) < 1,\quad \text{if } \|W_i\|_2  <1,
\end{align*}
where $\|\cdot\|_2$ denotes the spectral norm. Note, that regularizing the spectral norm of the Jacobian of $g$ \citep{Sokolic} only reduces it locally and does not guarantee the above condition. Thus, we will enforce $\|W_i\|_2  <1$ for each layer.

A power-iteration on the parameter matrix as in \citet{miyato2018spectral} approximates only a bound on $\|W_i\|_2$ instead of the true spectral norm, if the filter kernel is larger than $1\cross1$, see \citet{tsuzuku2018lipschitz} for details on the bound. Hence, unlike \citet{miyato2018spectral}, we directly estimate the spectral norm of $W_i$ by performing power-iteration using $W_i$ and $W_i^T$ as proposed in \citet{gouk}. The power-iteration yields an under-estimate $\tilde{\sigma}_i \leq \|W_i\|_2$. Using this estimate, we normalize via
\begin{align}
\label{eq:spectralNorm}
    \tilde{W}_i = \begin{cases}  c\, W_i / \tilde{\sigma}_i, \quad &\text{if } c / \tilde{\sigma}_i <1 \\
    W_i, \quad &\text{else}
    \end{cases},
\end{align}
where the hyper-parameter $c<1$ is a scaling coefficient. Since $\tilde{\sigma}_i$ is an under-estimate, $\|W_i\|_2  \leq c$ is not guaranteed. However, after training \citet{singularValConv} offer an approach to inspect $\|W_i\|_2$ exactly using the SVD on the Fourier transformed parameter matrix, which will allow us to show $\Lip (g) < 1$ holds in all cases.

\section{Generative Modelling with i-ResNets}
\label{sec:maxLikelihood}

We can define a simple generative model for data $x \in \mathbb{R}^d$ by first sampling $z \sim p_z(z)$ where $z \in \mathbb{R}^d$ and then defining $x = \Phi(z)$ for some function $\Phi: \mathbb{R}^d \rightarrow \mathbb{R}^d$. If $\Phi$ is invertible and we define $F = \Phi^{-1}$, then we can compute the likelihood of any $x$ under this model using the change of variables formula
\begin{align}
\label{eq:CoV}
    \ln p_x(x) = \ln p_z(z) + \ln |\det J_F(x)|,
\end{align}
where $J_F(x)$ is the Jacobian of $F$ evaluated at x. Models of this form are known as Normalizing Flows~\citep{rezende2015variational}. They have recently become a popular model for high-dimensional data due to the introduction of powerful bijective function approximators whose Jacobian log-determinant can be efficienty  computed~\citep{dinh2014nice, dinh2016density, kingma2018glow, chen2018neural} or approximated~\cite{ffjord}.

Since i-ResNets are guaranteed to be invertible we can use them to parameterize $F$ in Equation \eqref{eq:CoV}. Samples from this model can be drawn by first sampling $z \sim p(z)$ and then computing $x = F^{-1}(z)$ with Algorithm \ref{alg:inverse}. In Figure \ref{fig:toy2d} we show an example of using an i-ResNet to define a generative model on some two-dimensional datasets compared to Glow~\citep{kingma2018glow}.


\subsection{Scaling to Higher Dimensions}

While the invertibility of i-ResNets allows us to use them to define a Normalizing Flow, we must compute $\ln | \det J_F(x)|$ to evaluate the data-density under the model. Computing this quantity has a time cost of  $\mathcal{O}(d^3)$ in general which makes na\"{i}vely scaling to high-dimensional data impossible.

\begin{figure}
    \centering
    \begin{subfigure}[b]{0.33\linewidth}
    \centering
    \includegraphics[width=\linewidth]{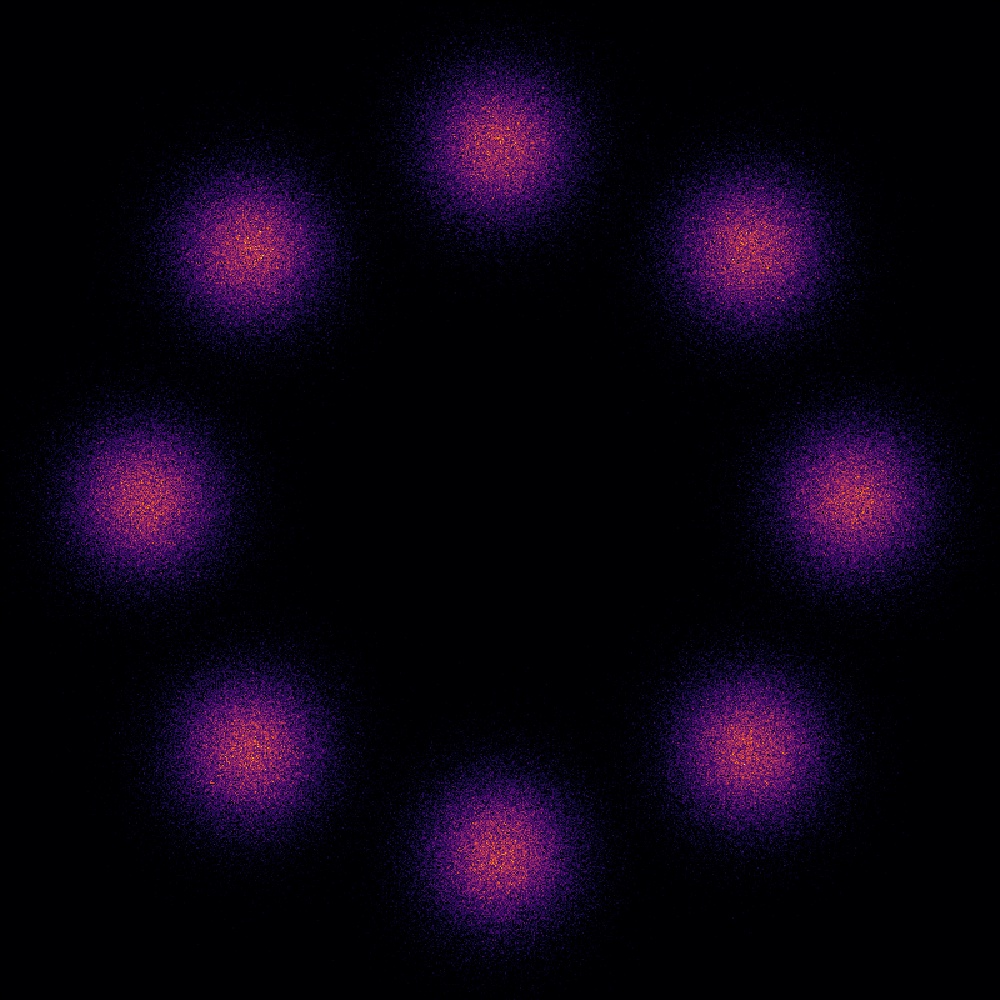}
    \end{subfigure}%
    \begin{subfigure}[b]{0.33\linewidth}
    \centering
    \includegraphics[width=\linewidth]{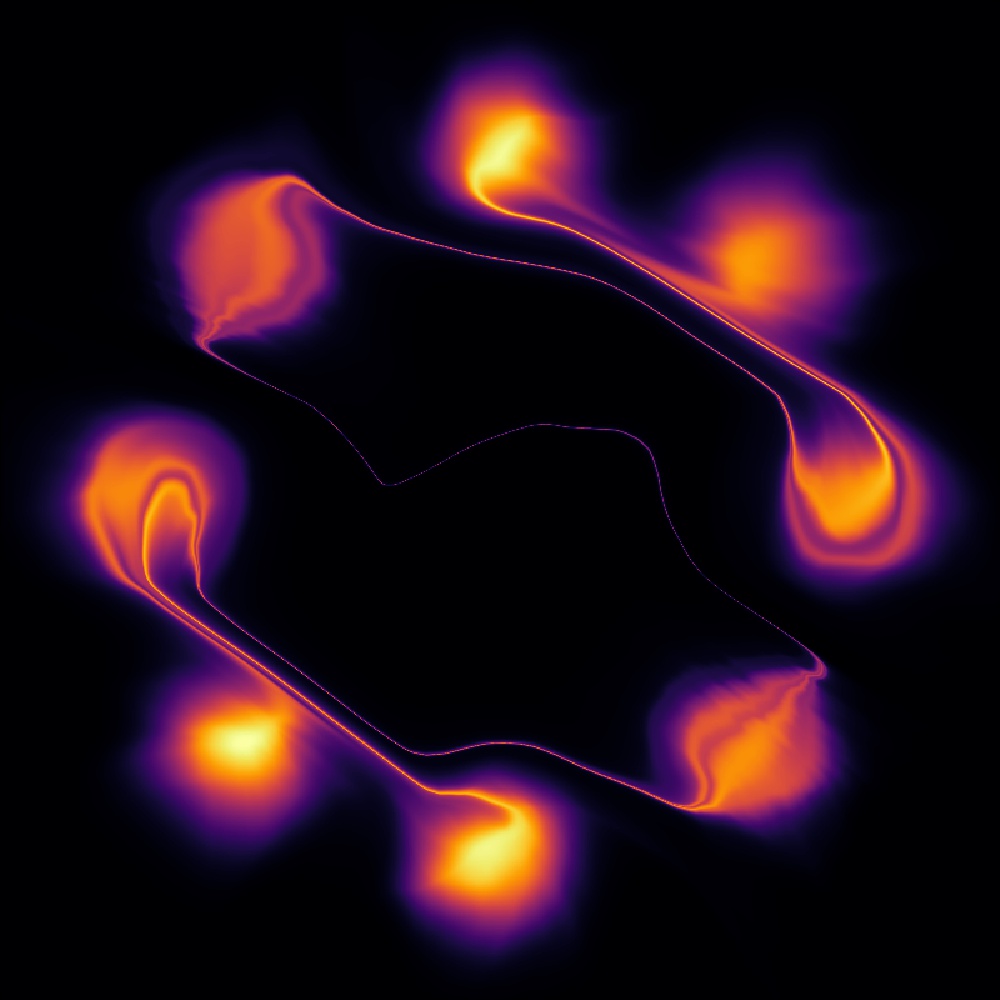}
    \end{subfigure}%
    \begin{subfigure}[b]{0.33\linewidth}
    \centering
    \includegraphics[width=\linewidth]{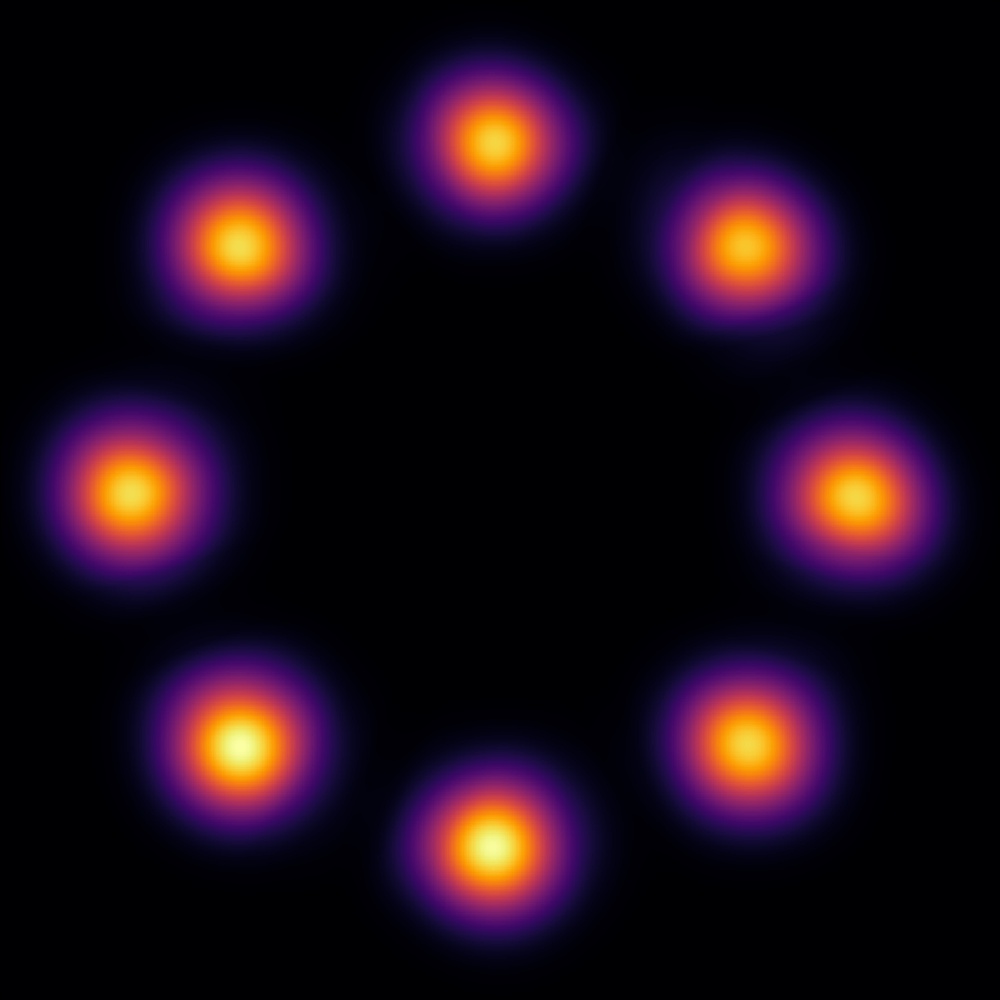}
    \end{subfigure}\\
    \begin{subfigure}[b]{0.33\linewidth}
    \centering
    \includegraphics[width=\linewidth]{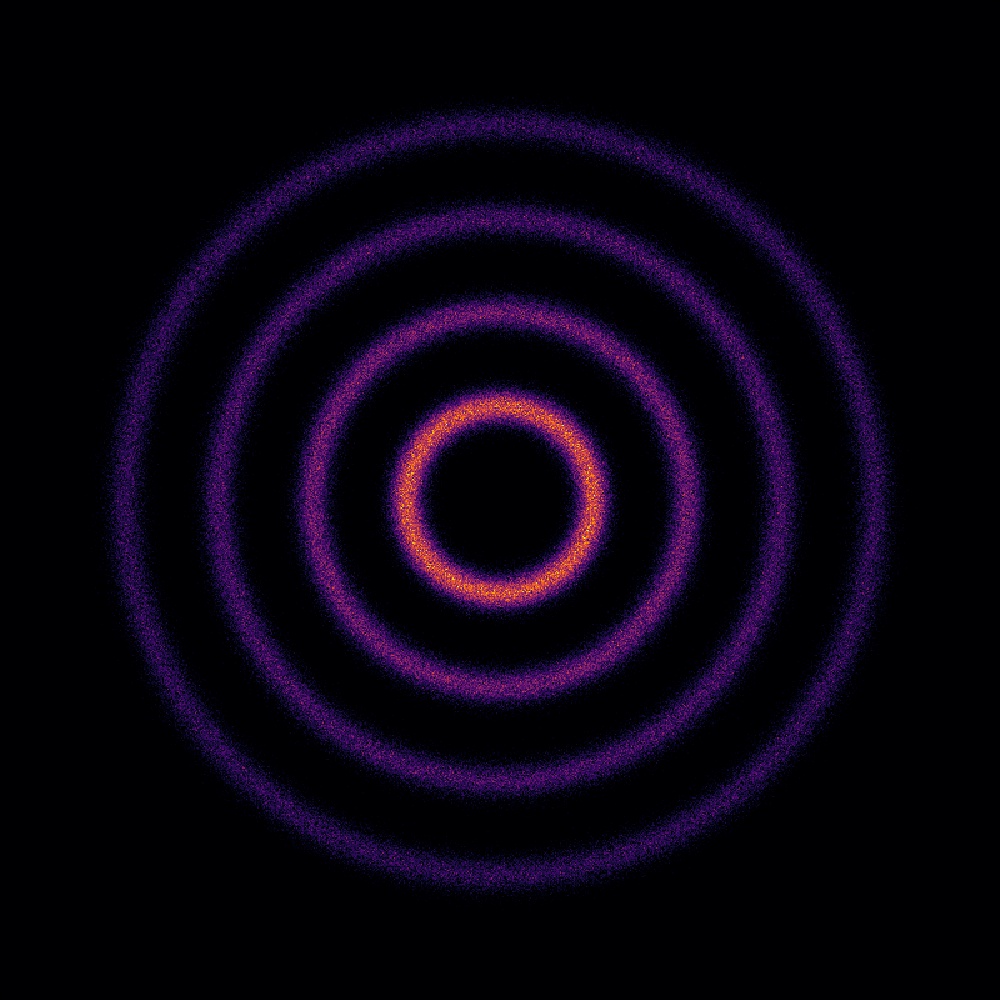}
    \caption*{Data Samples}
    \end{subfigure}%
    \begin{subfigure}[b]{0.33\linewidth}
    \centering
    \includegraphics[width=\linewidth]{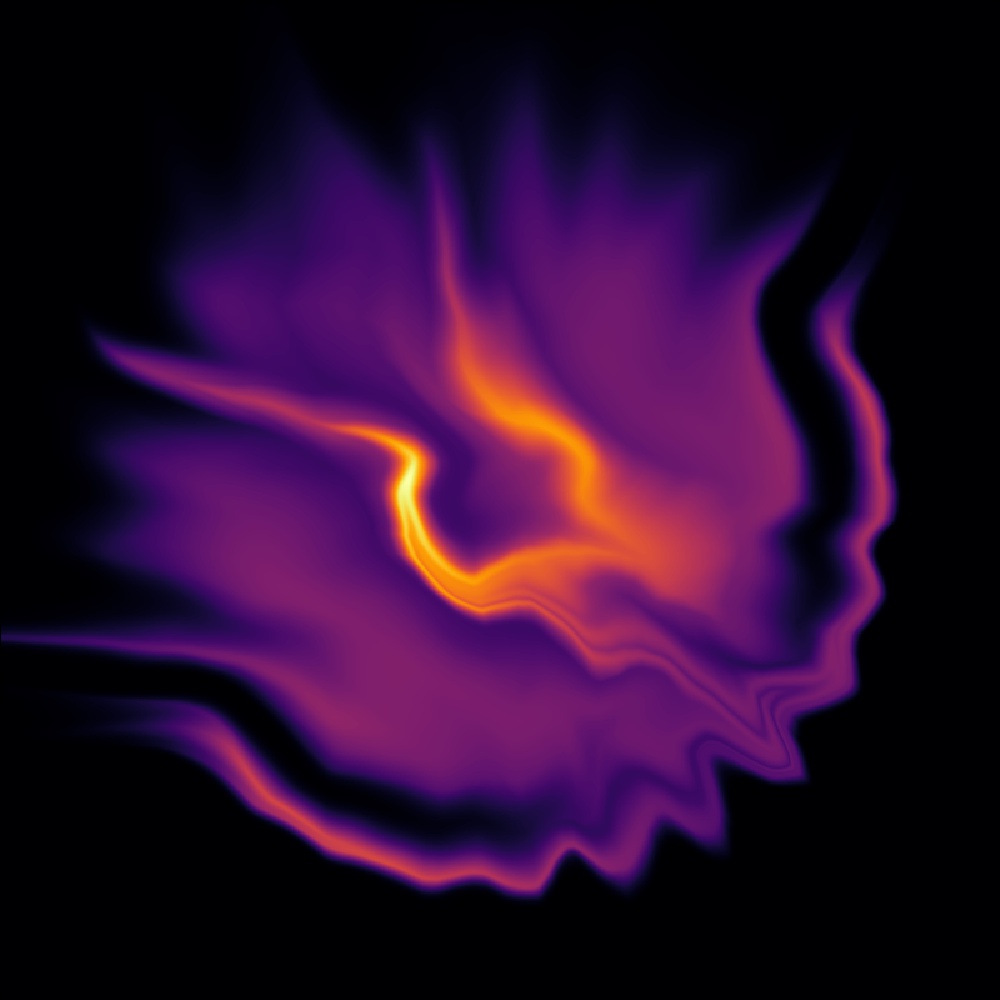}
    \caption*{Glow}
    \end{subfigure}%
    \begin{subfigure}[b]{0.33\linewidth}
    \centering
    \includegraphics[width=\linewidth]{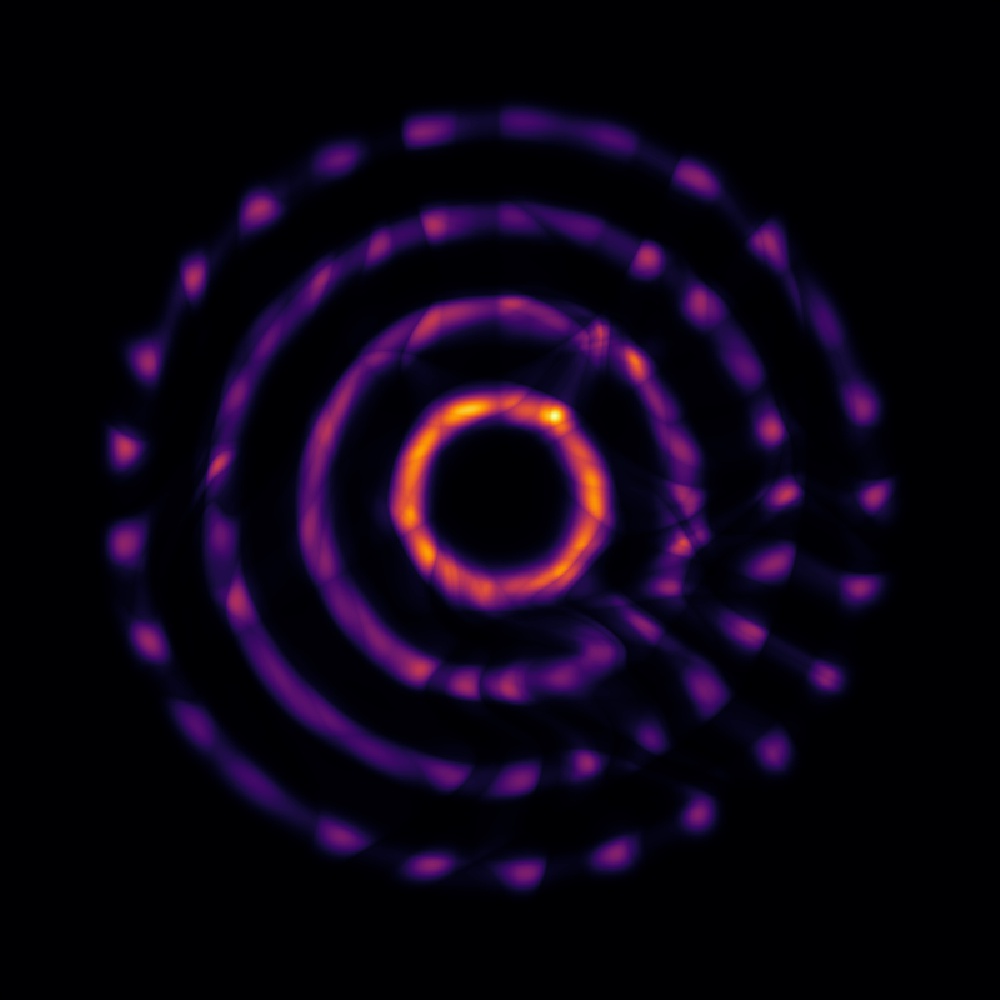}
    \caption*{i-ResNet}
    \end{subfigure}
    \caption{Visual comparison of i-ResNet flow and Glow. Details of this experiment can be found in Appendix C.3.}
    \label{fig:toy2d}
\end{figure}

To bypass this constraint we present a tractable approximation to the log-determinant term in Equation \eqref{eq:CoV}, which will scale to high dimensions $d$. Previously, \citet{leCunSpectralBackprop} introduced the application of log-determinant estimation to non-invertible deep generative models without the specific structure of i-ResNets.

First, we note that the Lipschitz constrained perturbations $x + g(x)$ of the identity yield positive determinants, hence
\begin{align*}
    |\det J_F(x) | = \det J_F(x),
\end{align*}
see Lemma 6 in Appendix A. Combining this result with the matrix identity $\ln \det (A)= \tr( \ln(A))$ for non-singular $A \in \mathbb{R}^{d \times d}$ (see e.g. \citet{Withers}), we have
\begin{align*}
    \ln |\det J_F(x)| = \tr( \ln J_F),
\end{align*}
where $\tr$ denotes the matrix trace and $\ln$ the matrix logarithm. Thus for $z= F(x) = (I + g) (x)$, it is
\begin{align*}
   \ln p_x(x) = \ln p_z(z) + \tr\left( \ln \big(I+ J_g (x)\big)\right).
\end{align*}
The trace of the matrix logarithm can be expressed as a power series \citep{Hall}
\begin{align}
\label{eq:powerSeriesLogDet}
   \tr\left(   \ln \big(I+ J_g (x)\big)\right) = \sum_{k=1}^\infty (-1)^{k+1}\frac{\tr(J_g^k)}{k},
\end{align}
which converges if $\|J_g\|_2 < 1$. Hence, due to the Lipschitz constraint, we can compute the log-determinant via the above power series with guaranteed convergence. 

Before we present a stochastic approximation to the above power series, we observe following properties of i-ResNets: Due to $\Lip(g_t) < 1$ for the residual block of each layer $t$, we can provide a lower and upper bound on its log-determinant with
\begin{align*}
    d \sum_{t=1}^T \ln (1 - \Lip(g_t)) \leq \ln |\det J_F(x) | \\
    d \sum_{t=1}^T \ln (1 + \Lip(g_t)) \geq \ln |\det J_F(x) |,
\end{align*}
for all $x\in \mathbb{R}$, see Lemma 7 in Appendix A. Thus, both the number of layers $T$ and the Lipschitz constant affect the contraction and expansion bounds of i-ResNets and must be taken into account when designing such an architecture.

\subsection{Stochastic Approximation of log-determinant}
\label{sec:approxLogDet}
Expressing the log-determinant with the power series in \eqref{eq:powerSeriesLogDet} has three main computational drawbacks: 1) Computing $\tr (J_g)$ exactly costs $\mathcal{O}(d^2)$, or approximately needs $d$ evaluations of $g$ as each entry of the diagonal of the Jacobian requires the computation of a separate derivative of $g$ \citep{ffjord}. 2) Matrix powers $J_g^k$ are needed, which requires the knowledge of the full Jacobian. 3) The series is infinite. 

Fortunately, drawback 1) and 2) can be alleviated. First, vector-Jacobian products $v^T J_g$ can be computed at approximately the same costs as evaluating $g$ through reverse-mode automatic differentiation. Second, a stochastic approximation of the matrix trace of $A \in \mathbb{R}^{d \times d}$
\begin{align*}
    \tr(A) = \mathbb{E}_{p(v)} \left[v^T A v \right],
\end{align*}
known as the Hutchinson´s trace estimator, can be used to estimate $\tr(J_g^k)$. The distribution $p(v)$ needs to fulfill $\mathbb{E}[v] = 0$ and $\text{Cov}(v) = I$, see \citep{Hutchinson, AvronToledo}.

While this allows for an unbiased estimate of the matrix trace, to achieve bounded computational costs, the power series \eqref{eq:powerSeriesLogDet} will be truncated at index $n$ to address drawback 3). Algorithm \ref{alg:algoOverview} summarizes the basic steps. The truncation turns the unbiased estimator into a biased estimator, where the bias depends on the truncation error. Fortunately, this error can be bounded as we demonstrate below.

To improve the stability of optimization when using this estimator we recommend using nonlinearities with continuous derivatives such as ELU~\citep{DBLP:journals/corr/ClevertUH15} or softplus instead of ReLU (See Appendix C.3).

\begin{algorithm}[tb]
   \caption{Forward pass of an invertible ResNets with Lipschitz constraint and log-determinant approximation, SN denotes spectral normalization based on \eqref{eq:spectralNorm}.}
   \label{alg:algoOverview}
\begin{algorithmic}
   \STATE {\bfseries Input:} data point $x$, network $F$, residual block $g$, number of power series terms $n$
   \FOR{Each residual block}
   \STATE Lip constraint: $\hat{W}_j := \text{SN}(W_j,x)$ for linear Layer $W_j$.
   \STATE Draw $v$ from $\mathcal{N}(0, I)$
   \STATE $w^T := v^T$ 
   \STATE $\ln \det := 0$
   \FOR{$k=1$ {\bfseries to} $n$}
   \STATE $w^T := w^T\, J_g$ (vector-Jacobian product)
   \STATE $\ln \det := \ln \det + (-1)^{k+1} w^T\, v / k$
   \ENDFOR
   \ENDFOR
\end{algorithmic}
\end{algorithm}

\subsection{Error of Power Series Truncation}
\label{sec:bound}
We estimate $\ln|\det( I + J_g)|$ with the finite power series
\begin{align}
    PS(J_g, n) := \sum_{k=1}^n (-1)^{k+1}\frac{\tr(J_g^k)}{k},
\end{align}
where we have (with some abuse of notation) $PS(J_g, \infty) = \tr(\ln(I + J_g))$. 
We are interested in bounding the truncation error of the log-determinant as a function of the data dimension $d$, the Lipschitz constant $\Lip(g)$ and the number of terms in the series $n$. 
\begin{theorem}[Approximation error of Loss]
\label{thm:biasLoss}
Let $g$ denote the residual function and $J_g$ the Jacobian as before. Then, the error of a truncated power series at term $n$ is bounded as
\begin{align*}
    & |PS(J_g, n) - \ln \det( I + J_g)| \\
    \leq & -d \left(\ln(1 - \Lip(g)) + \sum_{k=1}^n \frac{\Lip(g)^k}{k}\right).
\end{align*}
\end{theorem}
While the result above gives an error bound for evaluation of the loss, during training the error in the gradient of the loss is of greater interest. Similarly, we can obtain the following bound. The proofs are given in Appendix A.

\begin{theorem}[Convergence Rate of Gradient Approximation]
\label{thm:biasGrad}
Let $\theta \in \mathbb{R}^p$ denote the parameters of network $F$, let $g, J_g$ be as before. Further, assume bounded inputs and a Lipschitz activation function with Lipschitz derivative. Then, we obtain the convergence rate
\begin{align*}
    \|\nabla_\theta \left(\ln \det \big(I + J_g\big)) -  PS\big(J_g, n\big)\right)\|_\infty = \mathcal{O}(c^n)
\end{align*}
where $c := \Lip(g)$ and $n$ the number of terms used in the power series.
\end{theorem}

In practice, only 5-10 terms must be taken to obtain a bias less than .001 bits per dimension, which is typically reported up to .01 precision (See Appendix E).

\begin{table*}
\begin{center}
\begin{tabular}{S|SSSSS|S}
      {Method} & {ResNet} & {NICE/ i-RevNet} & {Real-NVP} & {Glow} & {FFJORD} & {i-ResNet} \\ \midrule
      {Free-form} & {\cmark}  & {\xmark} & {\xmark} & {\xmark} & {\cmark} & {\cmark} \\
      {Analytic Forward} & {\cmark}  & {\cmark} & {\cmark} & {\cmark} & {\xmark} & {\cmark}\\ 
      {Analytic Inverse} & {N/A}  & {\cmark}& {\cmark} & {\xmark} & {\xmark} & {\xmark}\\
      {Non-volume Preserving} & {N/A}  & {\xmark} & {\cmark} & {\cmark} & {\cmark} & {\cmark}\\
      {Exact Likelihood} & {N/A}  & {\cmark} & {\cmark} & {\cmark} & {\xmark} & {\xmark}\\
      {Unbiased Stochastic Log-Det Estimator} & {N/A} & {N/A} & {N/A} & {N/A} & {\cmark} & {\xmark} \\
      \bottomrule
\end{tabular}
\caption{Comparing i-ResNet and ResNets to NICE \citep{dinh2014nice}, Real-NVP \citep{dinh2016density}, Glow \citep{kingma2018glow} and FFJORD \citep{ffjord}. Non-volume preserving refers to the ability to allow for contraction and expansions and exact likelihood to compute the change of variables \eqref{eq:CoV} exactly. The unbiased estimator refers to a stochastic approximation of the log-determinant, see section \ref{sec:approxLogDet}.}
\label{tab:invertibleComparison}
\end{center}
\end{table*}

\section{Related Work}




\subsection{Reversible Architectures}

We put our focus on invertible architectures with efficient inverse computation, namely NICE~\citep{dinh2014nice}, i-RevNet~\citep{jacobsen2018irevnet}, Real-NVP~\citep{dinh2016density}, Glow~\citep{kingma2018glow} and Neural ODEs~\citep{chen2018neural} and its stochastic density estimator FFJORD~\citep{ffjord}. A summary of the comparison between different reversible networks is given in Table \ref{tab:invertibleComparison}.

The dimension-splitting approach used in NICE, i-RevNet, Real-NVP and Glow allows for both analytic forward and inverse mappings. However, this restriction required the introduction of additional steps like invertible $1\cross1$ convolutions in Glow \citep{kingma2018glow}. These $1\cross1$ convolutions need to be inverted numerically, making Glow altogether not analytically invertible. In contrast, i-ResNet can be viewed as an intermediate approach, where the forward mapping is given analytically, while the inverse can be computed via a fixed-point iteration. 

Furthermore, an i-ResNet block has a Lipschitz bound both for forward and inverse (Lemma \ref{thm:lipForwardInverse}), while other approaches do not have this property by design. Hence, i-ResNets could be an interesting avenue for stability-critical applications like inverse problems \citep{ardizzone2018analyzing} or invariance-based adversarial vulnerability \citep{jacobsen2018excessive}. 

Neural ODEs~\citep{chen2018neural} allow free-form dynamics similar to i-ResNets, meaning that any architecture could be used as long as the input and output dimensions are the same. To obtain discrete forward and inverse dynamics, Neural ODEs rely on adaptive ODE solvers, which allows for an accuracy vs. speed trade-off. Yet, scalability to very high input dimension such as high-resolution images remains unclear.

\subsection{Ordinary Differential Equations}
Due to the similarity of ResNets and Euler discretizations, there are many connections between the i-ResNet and ODEs, which we review in this section.


\textbf{Relationship of i-ResNets to Neural ODEs:} 
The view of deep networks as dynamics over time offers two fundamental learning approaches: 1) Direct learning of dynamics using discrete architectures like ResNets \citep{haberRuthotto, ruthottoHaber, luODE, naisnet}. 2) Indirect learning of dynamics via parametrizing an ODE with a neural network as in \citet{chen2018neural, ffjord}.

The dynamics $x(t)$ of a fixed ResNet $F_\theta$ are only defined at time
points $t_i$ corresponding to each block $g_{\theta_{t_i}}$. However, a linear interpolation in time can be used to generate continuous dynamics. See Figure \ref{fig:crossingPaths}, where the continuous dynamics of a linearly interpolated invertible ResNet are shown against those of a standard ResNet. Invertible ResNets are bijective along the continuous path while regular ResNets may result in crossing or merging paths. The indirect approach of learning an ODE, on the other hand, adapts the discretization based on an ODE-solver, but does not have a fixed computational budget compared to an i-ResNet. 

\textbf{Stability of ODEs:} There are two main approaches to study the stability of ODEs, 1) behavior for $t\rightarrow \infty$ and 2) Lipschitz stability over finite time intervals $[0,T]$. Based on time-invariant dynamics $f(x(t))$, \citep{naisnet} constructed asymptotically stable ResNets using anti-symmetric layers such that $Re(\lambda(J_x)) < 0$  (with $Re(\lambda(\cdot))$ denoting the real-part of eigenvalues, $\rho(\cdot)$ spectral radius and $J_x g$ the Jacobian at point x). By projecting weights based on the Gershgorin circle theorem, they further fulfilled $\rho(J_x g) < 1$, yielding asymptotically stable ResNets with shared weights over layers. On the other hand, \citep{haberRuthotto, ruthottoHaber} considered time-dependent dynamics $f(x(t), \theta(t))$ corresponding to standard ResNets. They induce stability by using anti-symmetric layers and projections of the weights.  Contrarily, initial value problems on $[0,T]$ are well-posed for Lipschitz continuous dynamics \citep{ascher2008numerical}. Thus, the invertible ResNet with $\Lip (f) < 1$ can be understood as a stabilizer of an ODE for step size $h=1$ without a restriction to anti-symmetric layers as in \citet{ruthottoHaber, haberRuthotto, naisnet}.

\subsection{Spectral Sum Approximations}

The approximation of spectral sums like the log-determinant is of broad interest for many machine learning problems such as Gaussian Process regression \citep{DongGPs}. Among others, Taylor approximation \citep{BOUTSIDIS201795} of the log-determinant similar to our approach or Chebyshev polynomials \citep{Han} are used. In \citet{BOUTSIDIS201795}, error bounds on the estimation via truncated power series and stochastic trace estimation are given for symmetric positive definite matrices. However, $I+ J_g$ is not symmetric and thus, their analysis does not apply here.

Recently, unbiased estimates \citep{adams} and unbiased gradient estimators \citep{HanGradients} were proposed for symmetric positive definite matrices. Furthermore, Chebyshev polynomials have been used to approximate the log-determinant of Jacobian of deep neural networks in \citet{leCunSpectralBackprop} for density matching and evaluation of the likelihood of GANs.

\section{Experiments}
\label{sec:exp}
We complete a thorough experimental survey of invertible ResNets. First, we numerically verify the invertibility of i-ResNets. Then, we investigate their discriminative abilities on a number of common image classification datasets. Furthermore, we compare the discriminative performance of i-ResNets to other invertible networks. Finally, we study how i-ResNets can be used to define generative models.

\subsection{Validating Invertibility and Classification}

To compare the discriminative performance and invertibility of i-ResNets with standard ResNet architectures, we train both models on CIFAR10, CIFAR100, and MNIST. 
The CIFAR and MNIST models have models have 54 and 21 residual blocks, respectively and we use identical settings for all other hyperparameters. We replace strided downsampling with ``invertible downsampling" operations~\cite{jacobsen2018irevnet} to ensure bijectivity, see Appendix C.2 for training and architectural details. We increase the number of input channels to 16 by padding with zeros. 
This is analagous to the standard practice of projecting the data into a higher-dimensional space using a standard convolutional layer at the input of a model, but this mapping is reversible.
To obtain the numerical inverse, we apply 100 fixed point iterations (Equation \eqref{eq:fixedpoint}) for each block.
This number is chosen to ensure that the poor reconstructions for vanilla ResNets (see Figure \ref{fig:reconstructionCifar}) are not due to using too few iterations. In practice far fewer iterations suffice, as the trade-off between reconstruction error and number of iterations analyzed in Appendix D shows. 

\begin{figure}
\includegraphics[width=1.\linewidth]{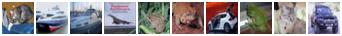}\\
\includegraphics[width=1.\linewidth]{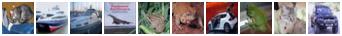}\\
\includegraphics[width=1.\linewidth]{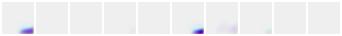}
\vspace{-7mm}
\caption{Original images (top) and reconstructions from i-ResNet with $c=0.9$ (middle) and a standard ResNet with the same architecture (bottom), showing that the fixed point iteration does not recover the input without the Lipschitz constraint.}
    \label{fig:reconstructionCifar}

\end{figure}

\begin{table*}
\begin{center}
\begin{tabular}{SSS|SSSSSSSS}
      {} & {} & {ResNet-164} & {Vanilla} &{$c = 0.9$}  & {$c = 0.8$}& {$c = 0.7$} & {$c = 0.6$}  & {$c = 0.5$} \\ \midrule   
    {\textbf{Classification}} & {MNIST} & {-} & {0.38}  & {0.40} & {0.42} & {0.40} & {0.42} &{0.86}   \\ 
    {Error \%} & {CIFAR10} & {5.50} & {6.69}  & {6.78} & {6.86} & {6.93} & {7.72} & {8.71}  \\ 
    {} & {CIFAR100} & {24.30} & {23.97}  & {24.58} & {24.99} & {25.99} & {27.30} &{29.45}   \\ \midrule
    {\textbf{Guaranteed Inverse}}& {} &{\textbf{No}} &{\textbf{No}}  & {\textbf{Yes}}& {\textbf{Yes}} & {\textbf{Yes}}  & {\textbf{Yes}} & {\textbf{Yes}} \\ \bottomrule  
\end{tabular}
\caption{Comparison of i-ResNet to a ResNet-164 baseline architecture of similar depth and width with varying Lipschitz constraints via coefficients $c$. Vanilla shares the same architecture as i-ResNet, without the Lipschitz constraint.}
\label{tab:iresnet164_cifar_classification}
\end{center}
\end{table*}

Classification and reconstruction results for a baseline pre-activation ResNet-164, a ResNet with architecture like i-ResNets without Lipschitz constraint (denoted as vanilla) and five invertible ResNets with different spectral normalization coefficients are shown in Table \ref{tab:iresnet164_cifar_classification}. The results illustrate that for larger settings of the layer-wise Lipschitz constant $c$, our proposed invertible ResNets perform competitively with the baselines in terms of classification performance, while being provably invertible. When applying very conservative normalization (small $c$), the classification error becomes higher on all datasets tested. 

To demonstrate that our normalization scheme is effective and that standard ResNets are not generally invertible, we reconstruct inputs from the features of each model using Algorithm \ref{alg:inverse}. Intriguingly, our analysis also reveals that unconstrained ResNets are invertible after training on MNIST (see Figure 7 in Appendix B), whereas on CIFAR10/100 they are not. Further, we find ResNets with and without BatchNorm are not invertible after training on CIFAR10, which can also be seen from the singular value plots in Appendix B (Figure 6). The runtime on 4 GeForce GTX 1080 GPUs with 1 spectral norm iteration was 0.5 sec for a forward and backward pass of batch with 128 samples, while it took 0.2 sec without spectral normalization. See section C.1 (appendix) for details on the runtime.

The reconstruction error decays quickly and the errors are already imperceptible after 5-20 iterations, which is the cost of 5-20 times the forward pass and corresponds to 0.15-0.75 seconds for reconstructing 100 CIFAR10 images.

Computing the inverse is fast even for the largest normalization coefficient, but becomes faster with stronger normalization. The number of iterations needed for full convergence is approximately cut in half when reducing the spectral normalization coefficient by 0.2, see Figure 8 (Appendix D) for a detailed plot. We also ran an i-RevNet \cite{jacobsen2018irevnet} with comparable hyperparameters as ResNet-164 and it performs on par with ResNet-164 with 5.6\%. Note however, that i-RevNets, like NICE~\cite{dinh2014nice}, are volume-conserving, making them less well-suited to generative modeling.

In summary, we observe that invertibility without additional constraints is unlikely, but possible, whereas it is hard to predict if networks will have this property. In our proposed model, we can guarantee the existence of an inverse without significantly harming classification performance.

\subsection{Comparison with Other Invertible Architectures}

\begin{table}[b]
\begin{center}
\begin{tabular}{SSS|S} \toprule
      {Affine Glow}& {Additive Glow}  & {i-ResNet} & {i-ResNet} \\ 
      {$1\cross1$ Conv}& {Reverse}& {Glow-Style} & {164} \\
      \midrule
    {12.63} & {12.36}  & {8.03}  &{6.69} 
\end{tabular}
\caption{CIFAR10 classification results compared to state-of-the-art flow Glow as a classifier.
We compare two versions of Glow, as well as an i-ResNet architecture as similar as possible to Glow in its number of layers and channels, termed ``i-ResNet, Glow-Style''.}
\label{tab:glow_comp}
\end{center}
\end{table}

In this section we compare i-ResNet classifiers to the state-of-the-art invertible flow-based model Glow. We take the implementation of \citet{kingma2018glow} and modify it to classify CIFAR10 images (with no generative modeling component). We create an i-ResNet that is as close as possible in structure to the default Glow model on CIFAR10 (denoted as i-ResNet Glow-style) and compare it to two variants of Glow, one that uses learned ($1 \cross 1$ convolutions) and affine block structure, and one with reverse permutations (like Real-NVP) and additive block structure. Results of this experiment can be found in Table \ref{tab:glow_comp}. We can see that i-ResNets outperform all versions of Glow on this discriminative task, even when adapting the network depth and width to that of Glow. This indicates that i-ResNets have a more suitable inductive bias in their block structure for discriminative tasks than Glow. 

We also find that i-ResNets are considerably easier to train than these other models. We are able to train i-ResNets using SGD with momentum and a learning rate of $0.1$ whereas all version of Glow we tested needed Adam or Adamax~\citep{kingma2014adam} and much smaller learning rates to avoid divergence. 

\subsection{Generative Modeling}
We run a number of experiments to verify the utility of i-ResNets in building generative models.
First, we compare i-ResNet Flows with Glow~\citep{kingma2018glow} on simple two-dimensional datasets.
Figure~\ref{fig:toy2d} qualitatively shows the density learned by a Glow model with 100 coupling layers and 100 invertible linear transformations. We compare against an i-ResNet where the coupling layers are replaced by invertible residual blocks with the same number of parameters and the invertible linear transformations are replaced by actnorm~\citep{kingma2018glow}. This results in the i-ResNet model having slightly fewer parameters, while maintaining an equal number of layers. In this experiment we train i-ResNets using the brute-force computed log-determinant since the data is two-dimensional. We find that i-ResNets are able to more accurately fit these simple densities.
As stated in \citet{ffjord}, we believe this is due to our model's ability to avoid partitioning dimensions. 

\begin{figure}[b]
    \centering
    \includegraphics[width=.95\linewidth]{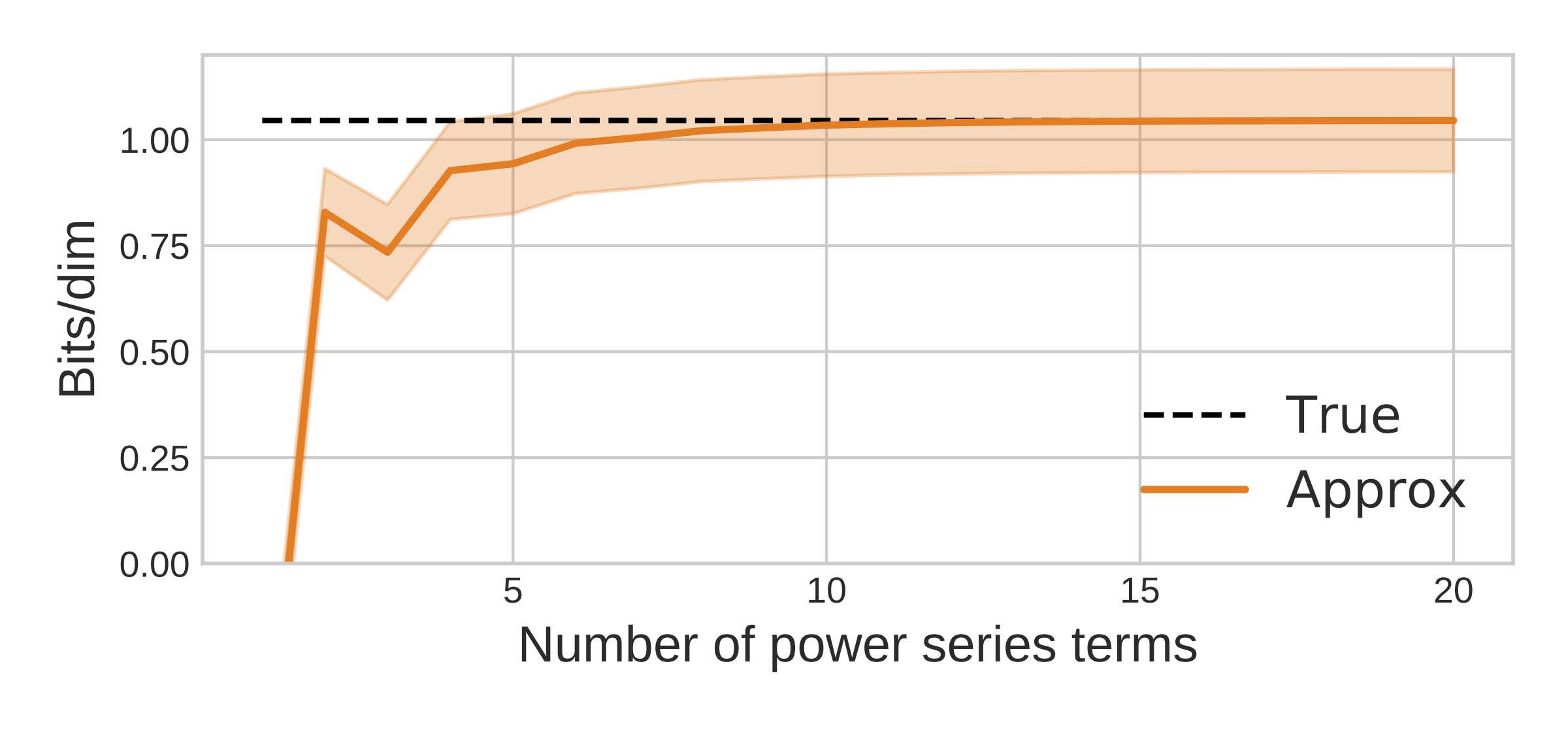}
    \vspace{-5mm}
    \caption{Bias and standard deviation of our log-determinant estimator as the number of power series terms increases. Variance is due to the stochastic trace estimator.}
    \label{fig:biasPlot}
\end{figure}

Next we evaluate i-ResNets as a generative model for images on MNIST and CIFAR10. Our models consist of multiple i-ResNet blocks followed by invertible downsampling or dimension ``squeezing'' to downsample the spatial dimensions. We use multi-scale architectures like those of \citet{dinh2016density, kingma2018glow}. In these experiments we train i-ResNets using the log-determinant approximation, see Algorithm \ref{alg:algoOverview}. Full architecture, experimental, and evaluation details can be found in Appendix C.3. Samples from our CIFAR10 model are shown in Figure \ref{fig:cifar_samples} and samples from our MNIST model can be found in Appendix F. Compared to the classification model, the log-determinant approximation with 5 series terms roughly increased the computation times by a factor of 4. The bias and variance of our log-determinant estimator is shown in Figure \ref{fig:biasPlot}.


Results and comparisons to other generative models can be found in Table \ref{tab:bitdim}. While our models did not perform as well as Glow and FFJORD, we find it intriguing that ResNets, with very little modification, can create a generative model competitive with these highly engineered models. We believe the gap in performance is mainly due to our use of a biased log-determinant estimator and that the use of an unbiased method~\cite{HanGradients} can help close this gap.  

\begin{figure}
    \centering
    \includegraphics[width=.99\linewidth]{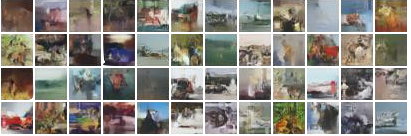}
    \caption{CIFAR10 samples from our i-ResNet flow. More samples can be found in Appendix F.}
    \label{fig:cifar_samples}
\end{figure}

\begin{table}
\begin{tabular}{@{}l|cc@{}}
Method & MNIST & CIFAR10 \\
\midrule
NICE~\citep{dinh2014nice} & 4.36 & 4.48$\dagger$\\
MADE~\citep{MADE} & 2.04 & 5.67 \\
MAF~\citep{MAF}  & 1.89 & 4.31 \\
Real NVP~\citep{dinh2016density} & 1.06 & 3.49 \\
Glow~\citep{kingma2018glow} & 1.05 & 3.35 \\
FFJORD~\citep{ffjord} & 0.99 & 3.40 \\
\midrule
i-ResNet & 1.06 & 3.45 
\end{tabular}
\caption{MNIST and CIFAR10 bits/dim results.  $\dagger$ Uses ZCA preprocessing making results not directly comparable.}
\label{tab:bitdim}
\end{table}

\section{Other Applications}

In many applications, a secondary unsupervised learning or generative modeling objective is formulated in combination with a primary discriminative task. i-ResNets are appealing here, as they manage to achieve competitive performance on both discriminative and generative tasks. We summarize some application areas to highlight that there is a wide variety of tasks for which i-ResNets would be promising to consider:
\begin{itemize}
    \item Hybrid density and discriminative models for joint classification and detection or fairness applications \citep{Nalisnick2018hybrid,louizos2015variational}
    \item Unsupervised learning for downstream tasks \citep{hjelm2018learning,oord2018representation}
    \item Semi-supervised learning from few labeled examples \citep{oliver2018realistic,kingma2014semi}
    \item Solving inverse problems with hybrid regression and generative losses \citep{ardizzone2018analyzing}
    \item Adversarial robustness with likelihood-based generative models \citep{schott2018towards,jacobsen2018excessive}
\end{itemize}

Finally, it is plausible that the Lipschitz bounds on the layers of the i-ResNet could aid with the stability of gradients for optimization, as well as adversarial robustness.

\section{Conclusions}
We introduced a new architecture, i-ResNets, which allow free-form layer architectures while still providing tractable density estimates.
The unrestricted form of the Jacobian allows expansion and contraction via the residual blocks, while partitioning-based models \citep{dinh2014nice, dinh2016density, kingma2018glow} must include affine blocks and scaling layers to be non-volume preserving.

Several challenges remain to be addressed in future work.
First, our estimator of the log-determinant is biased.
However, there have been recent advances in building unbiased estimators for the log-determinant~\citep{HanGradients}, which we believe could improve the performance of our generative model.
Second, learning and designing networks with a Lipschitz constraint is challenging.
For example, we need to constrain each linear layer in the block instead of being able to directly control the Lipschitz constant of a block, see \citet{sorting} for a promising approach for addressing this problem.


\section*{Acknowledgments}
We thank Rich Zemel for very helpful comments on an earlier version of the manuscript. We thank Yulia Rubanova for spotting a mistake in one of the proofs. We also thank everyone else at Vector for helpful discussions and feedback.

We gratefully acknowledge the financial support from the German Science Foundation for RTG 2224 "$\pi^3$: Parameter Identification - Analysis, Algorithms, Applications"

\bibliography{literature.bib}

\begin{thebibliography}{51}
\providecommand{\natexlab}[1]{#1}
\providecommand{\url}[1]{\texttt{#1}}
\expandafter\ifx\csname urlstyle\endcsname\relax
  \providecommand{\doi}[1]{doi: #1}\else
  \providecommand{\doi}{doi: \begingroup \urlstyle{rm}\Url}\fi

\bibitem[Adams et~al.(2018)Adams, Pennington, Johnson, Smith, Ovadia, Patton,
  and Saunderson]{adams}
Adams, R.~P., Pennington, J., Johnson, M.~J., Smith, J., Ovadia, Y., Patton,
  B., and Saunderson, J.
\newblock Estimating the spectral density of large implicit matrices.
\newblock \emph{arXiv preprint arXiv:1802.03451}, 2018.

\bibitem[Anil et~al.(2018)Anil, Lucas, and Grosse]{sorting}
Anil, C., Lucas, J., and Grosse, R.
\newblock Sorting out lipschitz function approximation.
\newblock \emph{arXiv preprint arXiv:1811.05381}, 2018.

\bibitem[Ardizzone et~al.(2019)Ardizzone, Kruse, Rother, and
  K\"othe]{ardizzone2018analyzing}
Ardizzone, L., Kruse, J., Rother, C., and K\"othe, U.
\newblock Analyzing inverse problems with invertible neural networks.
\newblock In \emph{International Conference on Learning Representations}, 2019.

\bibitem[Ascher(2008)]{ascher2008numerical}
Ascher, U.
\newblock \emph{Numerical methods for evolutionary differential equations}.
\newblock Computational science and engineering. Society for Industrial and
  Applied Mathematics, 2008.

\bibitem[Avron \& Toledo(2011)Avron and Toledo]{AvronToledo}
Avron, H. and Toledo, S.
\newblock Randomized algorithms for estimating the trace of an implicit
  symmetric positive semi-definite matrix.
\newblock \emph{J. ACM}, 58\penalty0 (2):\penalty0 8:1--8:34, 2011.

\bibitem[Boutsidis et~al.(2017)Boutsidis, Drineas, Kambadur, Kontopoulou, and
  Zouzias]{BOUTSIDIS201795}
Boutsidis, C., Drineas, P., Kambadur, P., Kontopoulou, E.-M., and Zouzias, A.
\newblock A randomized algorithm for approximating the log determinant of a
  symmetric positive definite matrix.
\newblock \emph{Linear Algebra and its Applications}, 533:\penalty0 95 -- 117,
  2017.

\bibitem[Chang et~al.(2018)Chang, Meng, Haber, Ruthotto, Begert, and
  Holtham]{Chang2018ReversibleAF}
Chang, B., Meng, L., Haber, E., Ruthotto, L., Begert, D., and Holtham, E.
\newblock Reversible architectures for arbitrarily deep residual neural
  networks.
\newblock \emph{Thirty-Second AAAI Conference on Artificial Intelligence},
  2018.

\bibitem[Chen et~al.(2018)Chen, Rubanova, Bettencourt, and
  Duvenaud]{chen2018neural}
Chen, T.~Q., Rubanova, Y., Bettencourt, J., and Duvenaud, D.
\newblock Neural ordinary differential equations.
\newblock \emph{Advances in Neural Information Processing Systems}, 2018.

\bibitem[Ciccone et~al.(2018)Ciccone, Gallieri, Masci, Osendorfer, and
  Gomez]{naisnet}
Ciccone, M., Gallieri, M., Masci, J., Osendorfer, C., and Gomez, F.
\newblock Nais-net: Stable deep networks from non-autonomous differential
  equations.
\newblock In \emph{Advances in Neural Information Processing Systems 31}, pp.\
  3029--3039. 2018.

\bibitem[Clevert et~al.(2015)Clevert, Unterthiner, and
  Hochreiter]{DBLP:journals/corr/ClevertUH15}
Clevert, D., Unterthiner, T., and Hochreiter, S.
\newblock Fast and accurate deep network learning by exponential linear units
  (elus).
\newblock \emph{arXiv preprint arXiv:1511.07289}, 2015.

\bibitem[Dinh et~al.(2014)Dinh, Krueger, and Bengio]{dinh2014nice}
Dinh, L., Krueger, D., and Bengio, Y.
\newblock Nice: Non-linear independent components estimation.
\newblock \emph{arXiv preprint arXiv:1410.8516}, 2014.

\bibitem[Dinh et~al.(2017)Dinh, Sohl-Dickstein, and Bengio]{dinh2016density}
Dinh, L., Sohl-Dickstein, J., and Bengio, S.
\newblock Density estimation using real nvp.
\newblock \emph{International Conference on Learning Representations}, 2017.

\bibitem[Dong et~al.(2017)Dong, Eriksson, Nickisch, Bindel, and
  Wilson]{DongGPs}
Dong, K., Eriksson, D., Nickisch, H., Bindel, D., and Wilson, A.~G.
\newblock Scalable log determinants for gaussian process kernel learning.
\newblock In \emph{Advances in Neural Information Processing Systems 30}, pp.\
  6327--6337. 2017.

\bibitem[Drucker \& Lecun(1992)Drucker and Lecun]{doubleBackprop}
Drucker, H. and Lecun, Y.
\newblock Improving generalization performance using double backpropagation.
\newblock \emph{IEEE transactions on neural networks}, 3:\penalty0 991--7, 02
  1992.

\bibitem[Germain et~al.(2015)Germain, Gregor, Murray, and Larochelle]{MADE}
Germain, M., Gregor, K., Murray, I., and Larochelle, H.
\newblock {MADE}: Masked autoencoder for distribution estimation.
\newblock In \emph{{ICML}}, volume~37 of \emph{{JMLR} Workshop and Conference
  Proceedings}, pp.\  881--889. JMLR.org, 2015.

\bibitem[Gomez et~al.(2017)Gomez, Ren, Urtasun, and
  Grosse]{gomez2017reversible}
Gomez, A.~N., Ren, M., Urtasun, R., and Grosse, R.~B.
\newblock The reversible residual network: Backpropagation without storing
  activations.
\newblock \emph{Advances in Neural Information Processing Systems}, 2017.

\bibitem[Gouk et~al.(2018)Gouk, Frank, Pfahringer, and Cree]{gouk}
Gouk, H., Frank, E., Pfahringer, B., and Cree, M.
\newblock Regularisation of neural networks by enforcing lipschitz continuity.
\newblock \emph{arXiv preprint arXiv:1804.04368}, 2018.

\bibitem[Grathwohl et~al.(2019)Grathwohl, Chen, Bettencourt, and
  Duvenaud]{ffjord}
Grathwohl, W., Chen, R. T.~Q., Bettencourt, J., and Duvenaud, D.
\newblock Ffjord: Scalable reversible generative models with free-form
  continuous dynamics.
\newblock In \emph{International Conference on Learning Representations}, 2019.

\bibitem[Haber \& Ruthotto(2018)Haber and Ruthotto]{haberRuthotto}
Haber, E. and Ruthotto, L.
\newblock Stable architectures for deep neural networks.
\newblock \emph{Inverse Problems}, 34\penalty0 (1):\penalty0 014004, 2018.

\bibitem[Hall(2015)]{Hall}
Hall, B.~C.
\newblock Lie groups, lie algebras, and representations: An elementary
  introduction.
\newblock \emph{Graduate Texts in Mathematics, 222 (2nd ed.), Springer}, 2015.

\bibitem[Han et~al.(2016)Han, Malioutov, Avron, and Shin]{Han}
Han, I., Malioutov, D., Avron, H., and Shin, J.
\newblock Approximating the spectral sums of large-scale matrices using
  chebyshev approximations.
\newblock \emph{SIAM Journal on Scientific Computing}, 39, 06 2016.

\bibitem[Han et~al.(2018)Han, Avron, and Shin]{HanGradients}
Han, I., Avron, H., and Shin, J.
\newblock Stochastic chebyshev gradient descent for spectral optimization.
\newblock In \emph{Advances in Neural Information Processing Systems 31}, pp.\
  7397--7407. 2018.

\bibitem[He et~al.(2016)He, Zhang, Ren, and Sun]{he2016deep}
He, K., Zhang, X., Ren, S., and Sun, J.
\newblock Deep residual learning for image recognition.
\newblock In \emph{Proceedings of the IEEE conference on computer vision and
  pattern recognition}, pp.\  770--778, 2016.

\bibitem[Hjelm et~al.(2019)Hjelm, Fedorov, Lavoie-Marchildon, Grewal, Bachman,
  Trischler, and Bengio]{hjelm2018learning}
Hjelm, R.~D., Fedorov, A., Lavoie-Marchildon, S., Grewal, K., Bachman, P.,
  Trischler, A., and Bengio, Y.
\newblock Learning deep representations by mutual information estimation and
  maximization.
\newblock In \emph{International Conference on Learning Representations}, 2019.

\bibitem[Hutchinson(1990)]{Hutchinson}
Hutchinson, M.
\newblock A stochastic estimator of the trace of the influence matrix for
  laplacian smoothing splines.
\newblock \emph{Communications in Statistics - Simulation and Computation},
  19\penalty0 (2):\penalty0 433--450, 1990.

\bibitem[Jacobsen et~al.(2018)Jacobsen, Smeulders, and
  Oyallon]{jacobsen2018irevnet}
Jacobsen, J.-H., Smeulders, A.~W., and Oyallon, E.
\newblock i-revnet: Deep invertible networks.
\newblock In \emph{International Conference on Learning Representations}, 2018.

\bibitem[Jacobsen et~al.(2019)Jacobsen, Behrmann, Zemel, and
  Bethge]{jacobsen2018excessive}
Jacobsen, J.-H., Behrmann, J., Zemel, R., and Bethge, M.
\newblock Excessive invariance causes adversarial vulnerability.
\newblock In \emph{International Conference on Learning Representations}, 2019.

\bibitem[Karras et~al.(2018)Karras, Laine, and Aila]{karras2018style}
Karras, T., Laine, S., and Aila, T.
\newblock A style-based generator architecture for generative adversarial
  networks.
\newblock \emph{arXiv preprint arXiv:1812.04948}, 2018.

\bibitem[Kingma \& Ba(2014)Kingma and Ba]{kingma2014adam}
Kingma, D.~P. and Ba, J.
\newblock Adam: A method for stochastic optimization.
\newblock \emph{arXiv preprint arXiv:1412.6980}, 2014.

\bibitem[Kingma \& Dhariwal(2018)Kingma and Dhariwal]{kingma2018glow}
Kingma, D.~P. and Dhariwal, P.
\newblock Glow: Generative flow with invertible 1x1 convolutions.
\newblock \emph{Advances in Neural Information Processing Systems}, 2018.

\bibitem[Kingma et~al.(2014)Kingma, Mohamed, Rezende, and
  Welling]{kingma2014semi}
Kingma, D.~P., Mohamed, S., Rezende, D.~J., and Welling, M.
\newblock Semi-supervised learning with deep generative models.
\newblock In \emph{Advances in neural information processing systems}, pp.\
  3581--3589, 2014.

\bibitem[Louizos et~al.(2016)Louizos, Swersky, Li, Welling, and
  Zemel]{louizos2015variational}
Louizos, C., Swersky, K., Li, Y., Welling, M., and Zemel, R.
\newblock The variational fair autoencoder.
\newblock \emph{International Conference on Learning Representations}, 2016.

\bibitem[Lu et~al.(2017)Lu, Zhong, Li, and Dong]{luODE}
Lu, Y., Zhong, A., Li, Q., and Dong, B.
\newblock Beyond finite layer neural networks: Bridging deep architectures and
  numerical differential equations.
\newblock \emph{arXiv preprint arXiv:1710.10121}, 2017.

\bibitem[Miyato et~al.(2018)Miyato, Kataoka, Koyama, and
  Yoshida]{miyato2018spectral}
Miyato, T., Kataoka, T., Koyama, M., and Yoshida, Y.
\newblock Spectral normalization for generative adversarial networks.
\newblock In \emph{International Conference on Learning Representations}, 2018.

\bibitem[Nalisnick et~al.(2018)Nalisnick, Matsukawa, Teh, Gorur, and
  Lakshminarayanan]{Nalisnick2018hybrid}
Nalisnick, E., Matsukawa, A., Teh, Y.~W., Gorur, D., and Lakshminarayanan, B.
\newblock Hybrid models with deep and invertible features.
\newblock \emph{NeurIPS workshop on Bayesian deep learning}, 2018.

\bibitem[Oliver et~al.(2018)Oliver, Odena, Raffel, Cubuk, and
  Goodfellow]{oliver2018realistic}
Oliver, A., Odena, A., Raffel, C.~A., Cubuk, E.~D., and Goodfellow, I.
\newblock Realistic evaluation of deep semi-supervised learning algorithms.
\newblock In \emph{Advances in Neural Information Processing Systems 31}. 2018.

\bibitem[Papamakarios et~al.(2017)Papamakarios, Murray, and Pavlakou]{MAF}
Papamakarios, G., Murray, I., and Pavlakou, T.
\newblock Masked autoregressive flow for density estimation.
\newblock In \emph{Advances in Neural Information Processing Systems}, 2017.

\bibitem[Parmar et~al.(2018)Parmar, Vaswani, Uszkoreit, Kaiser, Shazeer, Ku,
  and Tran]{DBLP:conf/icml/ParmarVUKSKT18}
Parmar, N., Vaswani, A., Uszkoreit, J., Kaiser, L., Shazeer, N., Ku, A., and
  Tran, D.
\newblock Image transformer.
\newblock In \emph{{ICML}}, volume~80 of \emph{{JMLR} Workshop and Conference
  Proceedings}, pp.\  4052--4061. JMLR.org, 2018.

\bibitem[Ramesh \& LeCun(2018)Ramesh and LeCun]{leCunSpectralBackprop}
Ramesh, A. and LeCun, Y.
\newblock Backpropagation for implicit spectral densities.
\newblock \emph{arXiv preprint arXiv:1806.00499}, 2018.

\bibitem[Rezende \& Mohamed(2015)Rezende and Mohamed]{rezende2015variational}
Rezende, D.~J. and Mohamed, S.
\newblock Variational inference with normalizing flows.
\newblock \emph{Proceedings of the 32nd International Conference on
  International Conference on Machine Learning}, 2015.

\bibitem[Ruthotto \& Haber(2018)Ruthotto and Haber]{ruthottoHaber}
Ruthotto, L. and Haber, E.
\newblock Deep neural networks motivated by partial differential equations.
\newblock \emph{arXiv preprint arXiv:1804.04272}, 2018.

\bibitem[Schott et~al.(2019)Schott, Rauber, Bethge, and
  Brendel]{schott2018towards}
Schott, L., Rauber, J., Bethge, M., and Brendel, W.
\newblock Towards the first adversarially robust neural network model on
  {MNIST}.
\newblock In \emph{International Conference on Learning Representations}, 2019.

\bibitem[Sedghi et~al.(2019)Sedghi, Gupta, and Long]{singularValConv}
Sedghi, H., Gupta, V., and Long, P.~M.
\newblock The singular values of convolutional layers.
\newblock In \emph{International Conference on Learning Representations}, 2019.

\bibitem[Sokolić et~al.(2017)Sokolić, Giryes, Sapiro, and Rodrigues]{Sokolic}
Sokolić, J., Giryes, R., Sapiro, G., and Rodrigues, M. R.~D.
\newblock Robust large margin deep neural networks.
\newblock \emph{IEEE Transactions on Signal Processing}, 65\penalty0
  (16):\penalty0 4265--4280, 2017.

\bibitem[Tsuzuku et~al.(2018)Tsuzuku, Sato, and Sugiyama]{tsuzuku2018lipschitz}
Tsuzuku, Y., Sato, I., and Sugiyama, M.
\newblock Lipschitz-margin training: Scalable certification of perturbation
  invariance for deep neural networks.
\newblock \emph{Advances in Neural Information Processing Systems}, 2018.

\bibitem[Van Den~Oord et~al.(2016{\natexlab{a}})Van Den~Oord, Dieleman, Zen,
  Simonyan, Vinyals, Graves, Kalchbrenner, Senior, and
  Kavukcuoglu]{van2016wavenet}
Van Den~Oord, A., Dieleman, S., Zen, H., Simonyan, K., Vinyals, O., Graves, A.,
  Kalchbrenner, N., Senior, A., and Kavukcuoglu, K.
\newblock Wavenet: A generative model for raw audio.
\newblock \emph{arXiv preprint arXiv:1609.03499}, 2016{\natexlab{a}}.

\bibitem[Van Den~Oord et~al.(2016{\natexlab{b}})Van Den~Oord, Kalchbrenner, and
  Kavukcuoglu]{oord2016pixel}
Van Den~Oord, A., Kalchbrenner, N., and Kavukcuoglu, K.
\newblock Pixel recurrent neural networks.
\newblock In \emph{Proceedings of the 33rd International Conference on
  International Conference on Machine Learning - Volume 48}, pp.\  1747--1756,
  2016{\natexlab{b}}.

\bibitem[Van Den~Oord et~al.(2018)Van Den~Oord, Li, and
  Vinyals]{oord2018representation}
Van Den~Oord, A., Li, Y., and Vinyals, O.
\newblock Representation learning with contrastive predictive coding.
\newblock \emph{arXiv preprint arXiv:1807.03748}, 2018.

\bibitem[Withers \& Nadarajah(2010)Withers and Nadarajah]{Withers}
Withers, C.~S. and Nadarajah, S.
\newblock log det a = tr log a.
\newblock \emph{International Journal of Mathematical Education in Science and
  Technology}, 41\penalty0 (8):\penalty0 1121--1124, 2010.

\bibitem[Zagoruyko \& Komodakis(2016)Zagoruyko and
  Komodakis]{zagoruyko2016wide}
Zagoruyko, S. and Komodakis, N.
\newblock Wide residual networks.
\newblock \emph{arXiv preprint arXiv:1605.07146}, 2016.

\bibitem[Zhao et~al.(2019)Zhao, Zhang, Sun, and Lee]{zhao2019information}
Zhao, T., Zhang, D., Sun, Z., and Lee, H.
\newblock Information regularized neural networks, 2019.
\newblock URL \url{https://openreview.net/forum?id=BJgvg30ctX}.

\end{thebibliography}
\bibliographystyle{icml2019}


\appendix
\onecolumn
\newpage

\section{Additional Lemmas and Proofs}
\label{sec:lemmasAndProofs}
\begin{proof}\textbf{(Theorem \ref{thm:invertibility})}\\
Since ResNet $F_\theta$ is a composition of functions, it is invertible if each block $F^t_{\theta}$ is invertible. Let $x_{t+1} \in \mathbb{R}^d$ be arbitrary and consider the backward Euler discretization $x_t = x_{t+1} - h f_{\theta_t} (x_t) = x_{t+1} - g_{\theta_t} (x_{t})$. Re-writing as a iteration yields
\begin{align}
x_t^0 := x_{t+1} \text{ and } x_t^{k+1} := x_{t+1} - g_{\theta_t}(x_t^{k}),
\end{align}
where $\lim_{k \rightarrow \infty} x_t^k = x_t$ is the fixed point if the iteration converges. As $g_{\theta_t}: \mathbb{R}^d \rightarrow \mathbb{R}^d$ is an operator on a Banach space, the contraction condition $\Lip(g_{\theta_t}) < 1$ guarantees convergence due to the Banach fixed point theorem.
\end{proof}
\begin{remark}
The condition above was also stated in \citet{zhao2019information} (Appendix D), however, their proof restricts the domain of the residual block $g$ to be bounded and applies only to linear operators $g$, because the inverse was given by a convergent Neumann-series.
\end{remark}

\begin{proof}\textbf{(Lemma \ref{thm:lipForwardInverse})}\\
First note, that $\Lip(F) \leq 1 + L$ follows directly from the addition of Lipschitz constants. For the inverse, consider
\begin{align}
\label{eq:reverseTriangle}
    \|F(x) - F(y)\|_2 \nonumber &= \|x-y +g(x) - g(y)\|_2 \\ \nonumber
    &= \|x-y-(-g(x)+g(y))\|_2 \\ 
    &\geq \left| \|x-y\|_2 - \|-g(x)+g(y)\|_2 \right|\\ \nonumber
    & \geq \left| \|x-y\|_2 - \|(-1)\, (g(x) - g(y))\|_2 \right| \\ \nonumber
    &\geq \|x-y\|_2 - \|g(x) - g(y)\|_2 \\ \nonumber
    &\geq \|x-y\|_2 - L\|x-y\|_2, \nonumber
\end{align}
where we apply the reverse triangular inequality in \eqref{eq:reverseTriangle} and apply the Lipschitz constant of $g$.
Denote $x = F^{-1}(z)$ and $y = F^{-1}(w)$ for $z, w \in \mathbb{R}^d$, which is possible since $F^{-1}$ is surjective. Inserting above yields
\begin{align*}
    \|F(F^{-1}(z)) - F(F^{-1}(w))\|_2 &\geq (1-L) \|F^{-1}(z)-F^{-1}(w)\|_2 \\
    \Longleftrightarrow \quad \quad \quad  \frac{1}{1-L}\|z - w\|_2 &\geq \|F^{-1}(z)-F^{-1}(w)\|_2 ,
\end{align*}
which holds for all $z, w$.
\end{proof}

\begin{lemma}[Positive Determinant of Jacobian of Residual Layer]
\label{thm:positiveDet}
Let $F(x) = (I + g(\cdot)) (x)$ denote a residual layer and $J_F(x) = I + J_g(x)$ its Jacobian at $x \in \mathbb{R}^d$. If $\Lip(g) < 1$, then it holds $\lambda_i$ of $J_F(x)$ are positive for all $x$ and thus
\begin{align*}
    | \det[J_F(x)] | = \det[J_F(x)],
\end{align*}
where $\lambda_i$ denotes the eigenvalues.
\end{lemma}
\begin{proof}\textbf{(Lemma \ref{thm:positiveDet})}\\
First, we have $\lambda_i(J_F) = \lambda_i(J_g) + 1$ and $\|J_g(x)\|_2 < 1$ for all $x$ due to $\Lip(g) < 1$. Since the spectral radius $\rho(J_g) \leq \|J_g\|_2$, it is $|\lambda_i(J_g)| < 1$. Hence, $Re(\lambda_i(J_F))>0$ and thus $\det J_F = \prod_i (\lambda_i(J_g) + 1) >0$.
\end{proof}

\begin{lemma}[Lower and Upper Bounds of an invertible ResNet on log-determinant]
\label{lem:lowerUpperBoundDeterminant}
Let $F_\theta: \mathbb{R}^d \rightarrow \mathbb{R}^d$ with $F_\theta = (F^1_{\theta} \circ \ldots \circ F^T_{\theta})$ denote an invertible ResNet with blocks $F^t_{\theta} = I + g_{\theta_t}$. Then, we can obtain the following bounds
\begin{align*}
    d \sum_{t=1}^T \ln (1 - \Lip(g_t)) \leq \ln |\det J_{F}(x) | \\
    d \sum_{t=1}^T \ln (1 + \Lip(g_t)) \geq \ln |\det J_{F}(x) |,
\end{align*}
for all $x \in \mathbb{R}^d$.
\end{lemma}
\begin{proof}\textbf{(Lemma \ref{lem:lowerUpperBoundDeterminant})})\\
First, the sum over the layers is due to the function composition, because $J_F(x) = \prod_t J_{F^t}(x)$ and 
\begin{align*}
    \ln |\det J_F (x)| = \ln \left(\prod_{t=1}^T \det J_{F^t}(x)\right) = \sum_{t=1}^T \ln \det J_{F^t}(x),
\end{align*}
where we use the positivity of the determinant, see Lemma \ref{thm:positiveDet}. Furthermore, note that 
\begin{align*}
\sigma_d(A)^d \leq \prod_i \sigma_i(A) = |\det A| \leq \sigma_1(A)^d    
\end{align*}
 for a matrix $A$ and largest singular values $\sigma_1$ and smallest $\sigma_d$. Furthermore, we have $\sigma_i(J_{F^t}) \leq (1 + \Lip(g_t))$ and $\sigma_d(J_{F^t}) \leq (1 - \Lip(g_t))$, which follows from Theorem \ref{thm:lipForwardInverse}. Inserting this and applying the logarithm rules finally yields the claimed bounds.
\end{proof}

\begin{proof}\textbf{(Theorem \ref{thm:biasLoss})}\\
We begin by noting that
\begin{align}
\label{eq:biasPowerSeries}
    |PS(J_g, n) - \tr \ln(J_g)| &= \left|\sum_{k=n+1}^\infty (-1)^{k+1} \frac{\tr(J_g^k)}{k}\right| \nonumber \\
    &\leq \sum_{k=n+1}^\infty \left|(-1)^{k+1}\frac{\tr(J_g^k)}{k} \right|\nonumber\\
    &\leq \sum_{k=n+1}^\infty \left|\frac{\tr(J_g^k)}{k} \right|\nonumber\\
    &\leq d \sum_{k=n+1}^\infty \frac{\Lip(g)^k}{k},
\end{align}
where inequality \eqref{eq:biasPowerSeries} follows from 
\begin{align}
\label{eq:traceSingEst}
    |\tr(J^k)| &\leq \left|\sum_{i=d}^d \lambda_i(J^k)\right| \leq \sum_{i=d}^d |\lambda_i(J^k)| \leq d \rho(J^k) \nonumber\\
    & \leq d \|J^k\|_2 \leq d \|J\|_2^k \leq d\; \Lip(g)^k.
\end{align}

We note that the full series $\sum_{k=1}^\infty \frac{\Lip(g)^k}{k} = -\ln(1 - \Lip(g))$ thus we can bound the approximation error by
\begin{align*}
    |PS(J_g, n) - \tr\ln(J_g)| \leq -d \left(\ln(1 - \Lip(g)) + \sum_{k=1}^n \frac{\Lip(g)^k}{k}\right)
\end{align*}
\end{proof}

\begin{proof}\textbf{(Theorem \ref{thm:biasGrad})}\\
First, we derive the by differentiating the power series and using the linearity of the trace operator. We obtain
\begin{align*}
\frac{\partial}{\partial \theta_i} \ln \det \big(I + J_g(x,\theta)\big) &= \frac{\partial}{\partial \theta_i} \left(\sum_{k=1}^\infty  (-1)^{k+1}\frac{\tr(J_g^k(x,\theta))}{k}\right)\\
&=  \tr \left(\sum_{k=1}^\infty   \frac{k(-1)^{k+1}}{k} J_g^{k-1}(x,\theta)\frac{\partial (J_g(x,\theta))}{\partial \theta_i} \right)\\
&=  \tr \left(\sum_{k=0}^\infty   (-1)^{k} J_g^{k}(x,\theta)\frac{\partial (J_g(x,\theta))}{\partial \theta_i} \right).
\end{align*}
By definition of $\|\cdot\|_\infty$, 
    \begin{align*}
        \|\nabla_\theta PS\big(J_g(\theta), \infty\big) - \nabla_\theta PS\big(J_g(\theta), n\big)\|_\infty 
        =  \max_{i=1,\ldots, p} \left|\frac{\partial}{\partial \theta_i} PS\big(J_g(\theta), \infty\big) - \frac{\partial}{\partial \theta_i} PS\big(J_g(\theta), n\big)\right|,
    \end{align*}
which is why, we consider an arbitrary $i$ from now on. It is 
\begin{align}
\label{eq:biasGrad}
    \left|\frac{\partial}{\partial \theta_i} PS\big(J_g(\theta), \infty\big) - \frac{\partial}{\partial \theta_i} PS\big(J_g(\theta), n\big)\right|
    =& \sum_{k=n+1}^\infty   (-1)^{k} \tr \left(J_g^{k}(x,\theta)\frac{\partial (J_g(x,\theta))}{\partial \theta_i}\right) \nonumber\\
    \leq & \;d \sum_{k=n+1}^\infty  \Lip(g)^K \left\|\frac{\partial J_g(x, \theta)}{\partial \theta_i}\right\|_2,
\end{align}
where we used the same arguments as in estimation \eqref{eq:traceSingEst}. 

In order to bound $\left\|\frac{\partial J_g(x, \theta)}{\partial \theta_i}\right\|_2$, we need to look into the design of the residual block. We assume contractive and element-wise activation functions (hence $\phi'(\cdot) < 1$) and $N$ linear layers $W_i$ in a residual block. Then, we can write the Jacobian as a matrix product
\begin{align*}
    J_g(x, \theta) = W_N^T D_N \cdots W_1^T D_1,
\end{align*}
where $D_i = diag(\phi'(z_{i-1}))$ with pre-activations $z_{i-1} \in \mathbb{R}^d$.

Since we need to bound the derivative of the Jacobian with respect to weights $\theta_i$, double backpropagation \citep{doubleBackprop} is necessary. In general, the terms $\|W_i^T\|_2$, $\|D_i\|_2$, $\|D^*_i\|_2 := \|diag(\phi''(z_{i-1}))\|_2$, $\left\|\left(\frac{\partial W_i}{\partial \theta_i}\right) \right\|_2$ and $\|x\|_2$ appear in the bound of the derivative. Hence, in order to bound $\left\|\frac{\partial J_g(x, \theta)}{\partial \theta_i}\right\|_2$, we bound the previous terms as follows
\begin{align}
    \|W_i^T\|_2 &\leq \Lip(g) \\
    \|D_i\|_2 &\leq const, \label{eq:contraction}\\
    \|D^*_i\|_2 &\leq const  \label{eq:contraction2} \\
    \|x\|_2 &\leq const \label{eq:boundedInput}\\
    \left\|\left(\frac{\partial W_i}{\partial \theta_i}\right)^T \right\|_2 &\leq  \|W_i\|_F + s  \label{eq:implementation}.
\end{align}
In particular, \eqref{eq:contraction} is due to the assumption of a Lipschitz activation function and \eqref{eq:contraction2} due to assuming a Lipschitz derivative of the activation function. Note, that we are using continuously differentiable activations functions (hence, not ReLU), where this assumptions holds for common functions like ELU, softplus and tanh. Furthermore, \eqref{eq:boundedInput} holds by assuming bounded inputs and due to the network being Lipschitz. To understand the bound \eqref{eq:implementation}, we denote $s$ as the amount of parameter sharing of $\theta_i$. For example, if $\theta_i$ is a entry from a convolution kernel, $s=w*h$ with $w$ spatial width and $h$ spatial height. Then
 \begin{align*}
      \left\|\left(\frac{\partial W_i}{\partial \theta_i}\right)^T \right\|_2 \leq \|W_i\|_F + s,
 \end{align*}
 since 
 \begin{align*}
\frac{\partial W_{lm}(x, \theta)}{\partial \theta_i} =
\begin{cases}
1, \quad \text{if }\; W_{lm}=\theta_i\\
0, \quad \text{else}\\
\end{cases}.
\end{align*}
Hence, as each term appearing in the second derivative $\left\|\frac{\partial J_g(x, \theta)}{\partial \theta_i}\right\|_2$ is bounded, we can introduce the constant $a(g, \theta,x) < \infty$ which depends on the parameters, the implementation of $g$ and the inputs $x$. Note, that we do not give an exact bound on $\left\|\frac{\partial J_g(x, \theta)}{\partial \theta_i}\right\|_2$, since we are only interesting in the existence of such a bound in order to proof the convergence in the claim.

Inserting above in \eqref{eq:biasGrad} and denoting $c:=\Lip(g)$, yields
\begin{align*}
    \left|\frac{\partial}{\partial \theta_i} PS\big(J_g(\theta), \infty\big) - \frac{\partial}{\partial \theta_i} PS\big(J_g(\theta), n\big)\right|
    \leq  \;d\; a(g, \theta, x)  \sum_{k=n+1}^\infty c^K = \tilde{a}(d, g, \theta, x)  \left(\frac{1}{1-c} - \left( \frac{1-c^n}{1-c}+ c^n\right)\right) . 
\end{align*}
Letting $f(n) := \tilde{a}(d, g, \theta, x)   \left(\frac{1}{1-c} - \left( \frac{1-c^n}{1-c}+ c^n\right)\right)$ and $g(n) = c^n$, then 
\begin{align*}
    \lim_{n \rightarrow \infty} \left|\frac{f(n)}{g(n)} \right| = const < \infty,
\end{align*}
which proves the claimed convergence rate.
\end{proof}

\section{Verification of Invertibility}
\paragraph{Inveritibility of Learned Mappings} In this experiment we train standard ResNets and i-ResNets with various layer-wise Lipschitz coefficients ($c \in \{.3, .5, .7, .9\})$. After training, we inspect the learned transformations at each layer by computing the largest singular value of each linear mapping based on the approach in \citet{singularValConv}. It can be seen clearly (Figure \ref{fig:singularvalues} left) that the standard and BatchNorm models have many singular values above 1, making their residual connections non-invertible. Conversely, in the i-ResNet models (Figure \ref{fig:singularvalues} right), all singular values are below 1 (and roughly equal to $c$) indicating their residual connections are invertible. 
\label{sec:invver}
\begin{figure}[H]
    {{\includegraphics[width=0.44\linewidth]{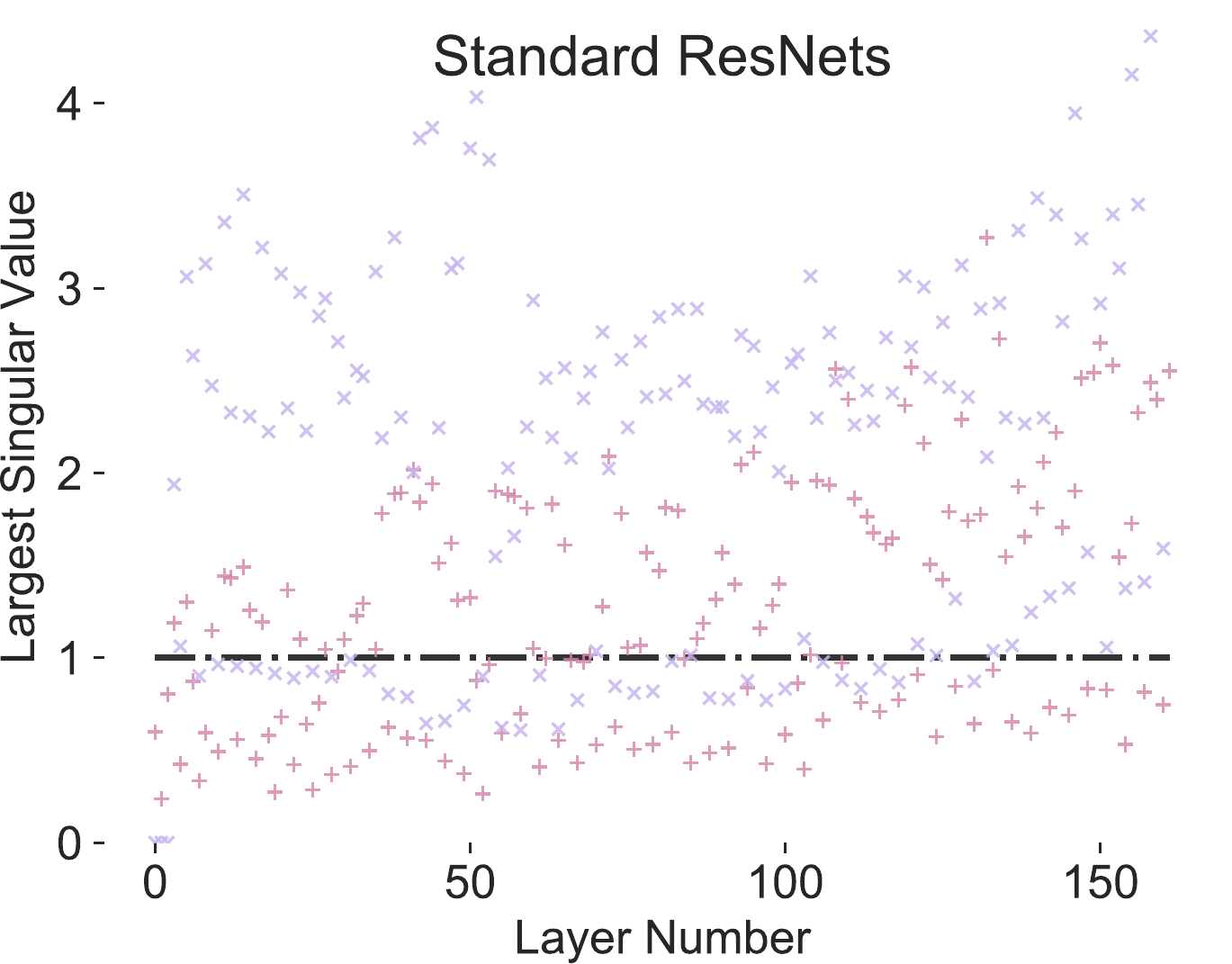} }}%
    \,
    {{\includegraphics[width=0.425\linewidth]{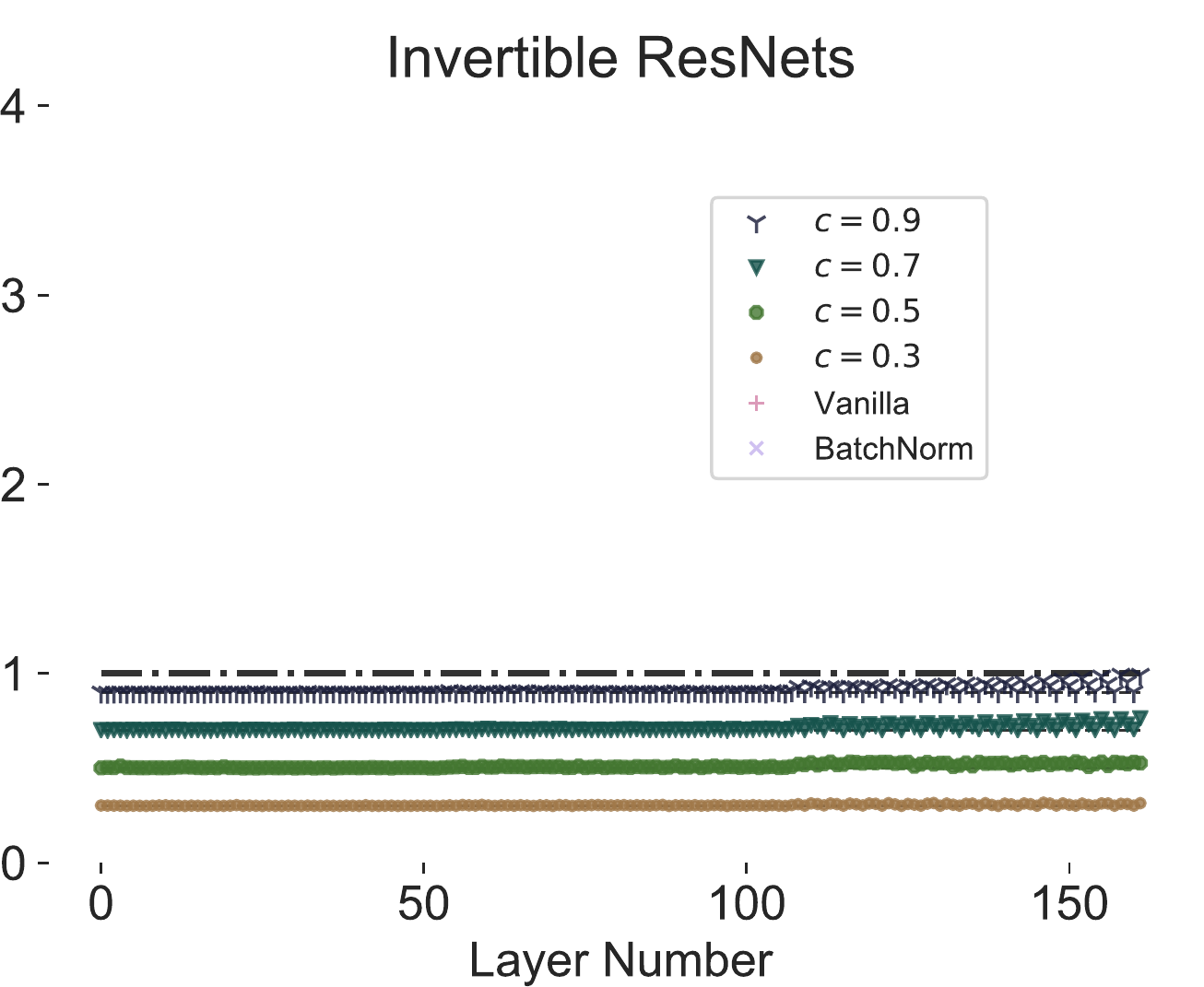} }}%
    \vspace{-2mm}
    \caption{Maximal singular value of each layers convolutional operator for various trained ResNets on Cifar10.  \textit{Left:} Vanilla and Batchnorm ResNet singular values. It is likely that the baseline ResNets are not invertible as roughly two thirds of their layers have singular values fairly above one, making the blocks non-contractive. \textit{Right:} Singular values for our 4 spectrally normalized ResNets. The regularization is effective and in every case the single ResNet block remains a contraction.}%
    \label{fig:singularvalues}
\end{figure}

\paragraph{Computing Inverses with Fixed-Point Iteration} Here we numerically compute inverses in our trained models using the fixed-point iteration, see Algorithm \ref{alg:inverse}. We invert each residual connection using 100 iterations (to ensure convergence). We see that i-ResNets can be inverted using this method whereas with standard ResNets this is not guaranteed (Figure \ref{fig:reconstruction} top). Interestingly, on MNIST we find that standard ResNets are indeed invertible after training on MNIST (Figure \ref{fig:reconstruction} bottom).  

\begin{figure}[H]
\begin{tabular}{R{3.9cm}L{9cm}}
{CIFAR Data:}&
\includegraphics[width=.65\linewidth]{img/dataCifar.jpg}\\
{i-ResNet Reconstructions with $c=0.9$:}&
\includegraphics[width=.65\linewidth]{img/reconsCifarCoeff09}\\
{Vanilla ResNet Reconstructions}&
    \includegraphics[width=.65\linewidth]{img/reconsCifarCoeff1000-resNet164.jpg}\\\midrule
{MNIST Data:}&
\includegraphics[width=.65\linewidth]{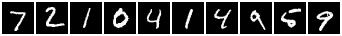}\\
{i-ResNet Reconstructions with $c=0.9$:}&
\includegraphics[width=.65\linewidth]{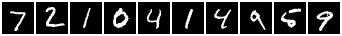} \\
{Vanilla ResNet Reconstructions:}&
\includegraphics[width=.65\linewidth]{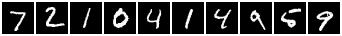}
\end{tabular}
    \caption{Original images (top), i-ResNets with $c=0.9$ (middle) and reconstructions from vanilla (bottom). Surprisingly, MNIST reconstructions are close to exact for both models, even without explicitly enforcing the Lipschitz constant. On CIFAR10 however, reconstructions completely fail for the vanilla ResNet, but are qualitatively and quantitatively exact for our proposed network. }
    \label{fig:reconstruction}
\end{figure}

\section{Experimental Details}

\subsection{Runtime Comparison}
\label{sec:runtime}
\begin{table}[h!]
\begin{center}
\begin{tabular}{S|SSS} \toprule
      {Glow}& {i-ResNet}  & {i-ResNet SN} & {i-ResNet SN LogDet} \\ 
      \midrule
    {0.72 sec} & {0.31 sec}  & {0.57 sec}  &{1.88 sec} 
\end{tabular}
\caption{Timings for a forward and backward pass. Glow is the dimension-splitting baseline, it has the same number of layers and channels as all i-ResNets and batch size is identical. We compare Glow to: 1) plain, 2) spectral norm, 3) spectral norm and log determinant estimated i-ResNets. Discriminative i-ResNets are around 1.2 - 2.1 times faster, while generative i-ResNets are around 2.6 times slower than the dimension splitting baseline. Thus, overall wall-clock times do not differ by much more than a factor of two in all cases.}
\label{tab:runtimeComp}
\end{center}
\end{table}

\subsection{Classification}
\label{sec:class_details}
\paragraph{Architecture} We use pre-activation ResNets with 39 convolutional bottleneck blocks with 3 convolution layers each and kernel sizes of 3x3, 1x1, 3x3 respectively. All models use the ELU nonlinearity~\citep{DBLP:journals/corr/ClevertUH15}. In the BatchNorm version, we apply batch normalization before every nonlinearity and in the invertible models we use ActNorm~\citep{kingma2018glow} before each residual block. The network has 2 down-sampling stages after 13 and 26 blocks where a dimension squeezing operation is used to decrease the spatial resolution. This reduces the spatial dimension by a factor of two in each direction, while increasing the number of channels by a factor of four. All models transform the data to a 8x8x256 tensor to which we apply BatchNorm, a nonlinearity, and average pooling to a 256-dimensional vector. A linear classifier is used on top of this representation. 

\paragraph{Injective Padding} Since our invertible models are not able to increase the dimension of their latent representation, we use injective padding~\citep{jacobsen2018irevnet} which concatenates channels of 0's to the input, increasing the size of the transformed tensor. This is analagous to the standard practice of projecting the data into a higher-dimensional space using a non-ResNet convolution at the input of a model, but this mapping is reversible. We add 13 channels of 0's to all models tested, thus the input to our first residual block is a tensor of size 32x32x16. We experimented with removing this step but found it led to approximately a 2\% decrease in accuracy for our CIFAR10 models. 


\paragraph{Training} We train for 200 epochs with momentum SGD and a weight decay of 5e-4. The learning rate is set to 0.1 and decayed by a factor of 0.2 after 60, 120 and 160 epochs. For data-augmentation, we apply random shifts of upt to two pixels for MNIST and shifts/ random horizontal flips for CIFAR(10/100) during training. The inputs for MNIST are normalized to [-0.5,0.5] and for CIFAR(10/100) normalize by subtracting the mean and dividing by the standard deviation of the training set. 

\subsection{Generative Modeling}
\label{sec:gen_details}
\paragraph{Toy Densities} We used 100 residual blocks, where each residual connection is a multilayer perceptron with state sizes of 2-64-64-64-2 and ELU nonlinearities \citep{DBLP:journals/corr/ClevertUH15}. We used ActNorm \citep{kingma2018glow} after each residual block. The change in log density was computed exactly by constructing the full Jacobian during training and visualization.

\paragraph{MNIST and CIFAR}
The structure of our generative models closely resembles that of Glow. The model consists of ``scale-blocks'' which are groups of i-ResNet blocks that operate at different spatial resolutions. After each scale-block, apart from the last, we perform a squeeze operation which decreases the spatial resolution by 2 in each dimension and multiplies the number of channels by 4 (invertible downsampling). 

Our MNIST and CIFAR10 models have three scale-blocks. Each scale-block has 32 i-ResNet blocks. Each i-ResNet block consists of three convolutions of $3 \cross 3$, $1 \cross 1$, $3 \cross 3$ filters with ELU~\citep{DBLP:journals/corr/ClevertUH15} nonlinearities in between. Each convolutional layer has 32 filters in the MNIST model and 512 filters in the CIFAR10 model.

We train for 200 epochs using the Adamax~\citep{kingma2014adam} optimizer with a learning rate of $.003$. Throughout training we estimate the log-determinant in Equation \eqref{eq:CoV} using the power-series approximation (Equation \eqref{eq:powerSeriesLogDet}) with ten terms for the MNIST model and 5 terms for the CIFAR10 model. 

\paragraph{Evaluation} During evaluation we use the bound presented in Section \ref{sec:bound} to determine the number of terms needed to give an estimate with bias less than .0001 bit/dim. We then average over enough samples from Hutchinson's estimator such that the standard error is less than .0001 bit/dim, thus we can safely report our model's bit/dim accurate up to a tolerance of .0002.

\paragraph{Choice of Nonlinearity} Differentiating our log-determinant estimator requires us to compute second derivatives of our neural network's output. If we were to use a nonlinearity with discontinuous derivatives (i.e. ReLU), then these values are not defined in certain regions. This can lead to unstable optimization. To guarantee the quantities required for optimization always exist, we recommend using nonlinearities which have continuous derivatives such as ELU~\citep{DBLP:journals/corr/ClevertUH15} or softplus. In all of our experiments we use ELU.

\section{Fixed Point Iteration Analysis}
\label{app:iterationtradeoff}

\begin{figure}[H]
    \centering
    {{\includegraphics[width=.4\linewidth]{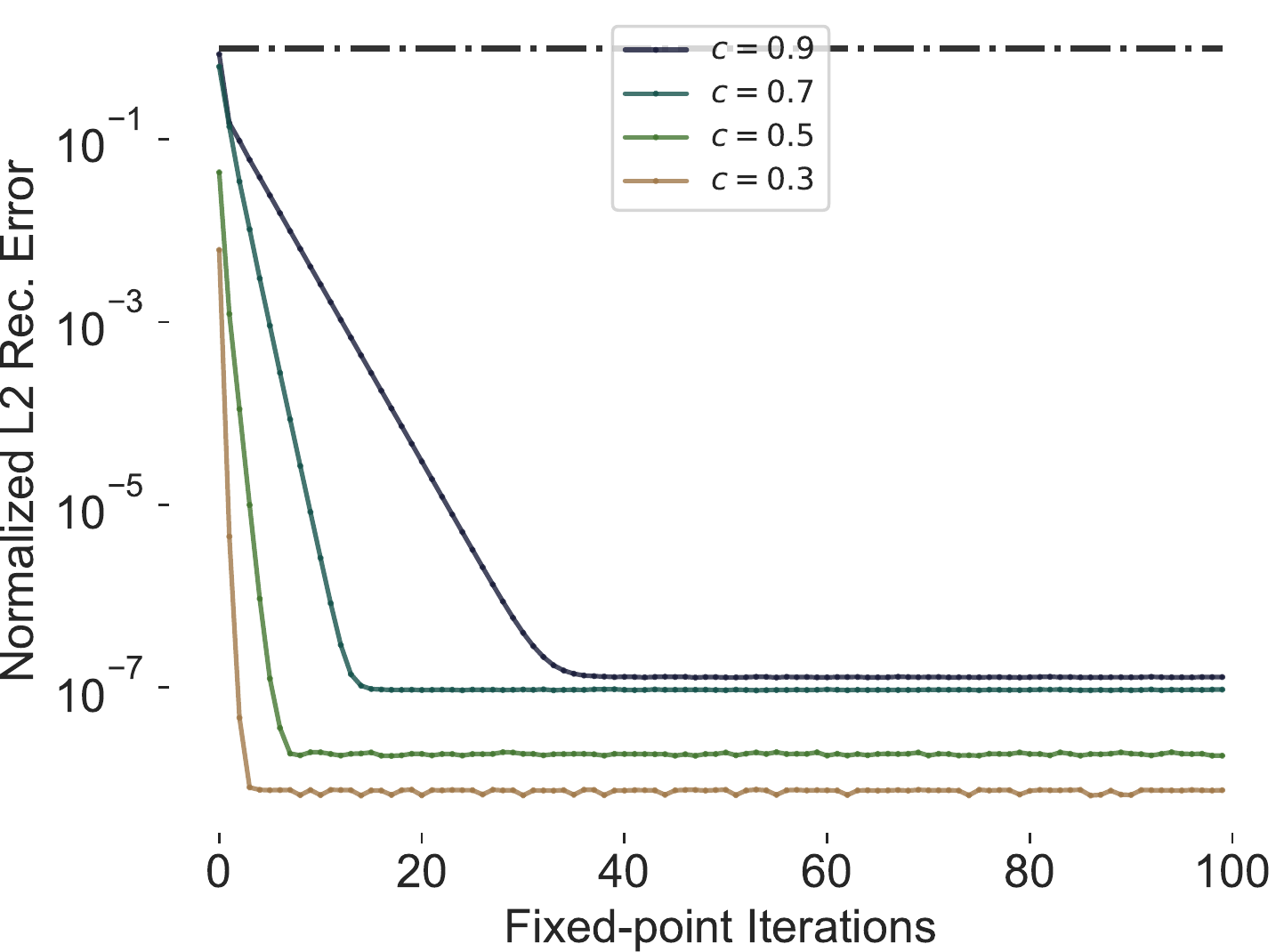} }}%
    \caption{Trade-off between number of fixed point iterations and reconstruction error (log scale) for computing the inverse for different normalization coefficients of trained invertible ResNets (on CIFAR10). The reconstruction error decays quickly. 5-20 iterations are sufficient respectively to obtain visually perfect reconstructions. Note that one iteration corresponds to the time for one forward pass, thus inversion is approximately 5-20 times slower than inference. This corresponds to a reconstruction time of 0.15-0.75 seconds for a batch of 100 CIFAR10 images with 5-20 iterations and 4.3 seconds with 100 iterations.}%
    \label{fig:reconError}%
\end{figure}

\section{Evaluating the Bias of Our Log-determinant Estimator}
\label{sec:numbias}
Here we numerically evaluate the bias of the log-determinant estimator used to train our generative models (Equation \eqref{eq:powerSeriesLogDet}). We compare the true value (computed via brute-force) with the estimator's mean and standard deviation as the number of terms in the power series is increased. After 10 terms, the estimator's bias is negligible and after 20 terms it is numerically 0. This is averaged over 1000 test examples.
\begin{figure}[H]
    \centering
    {{\includegraphics[width=.45\linewidth]{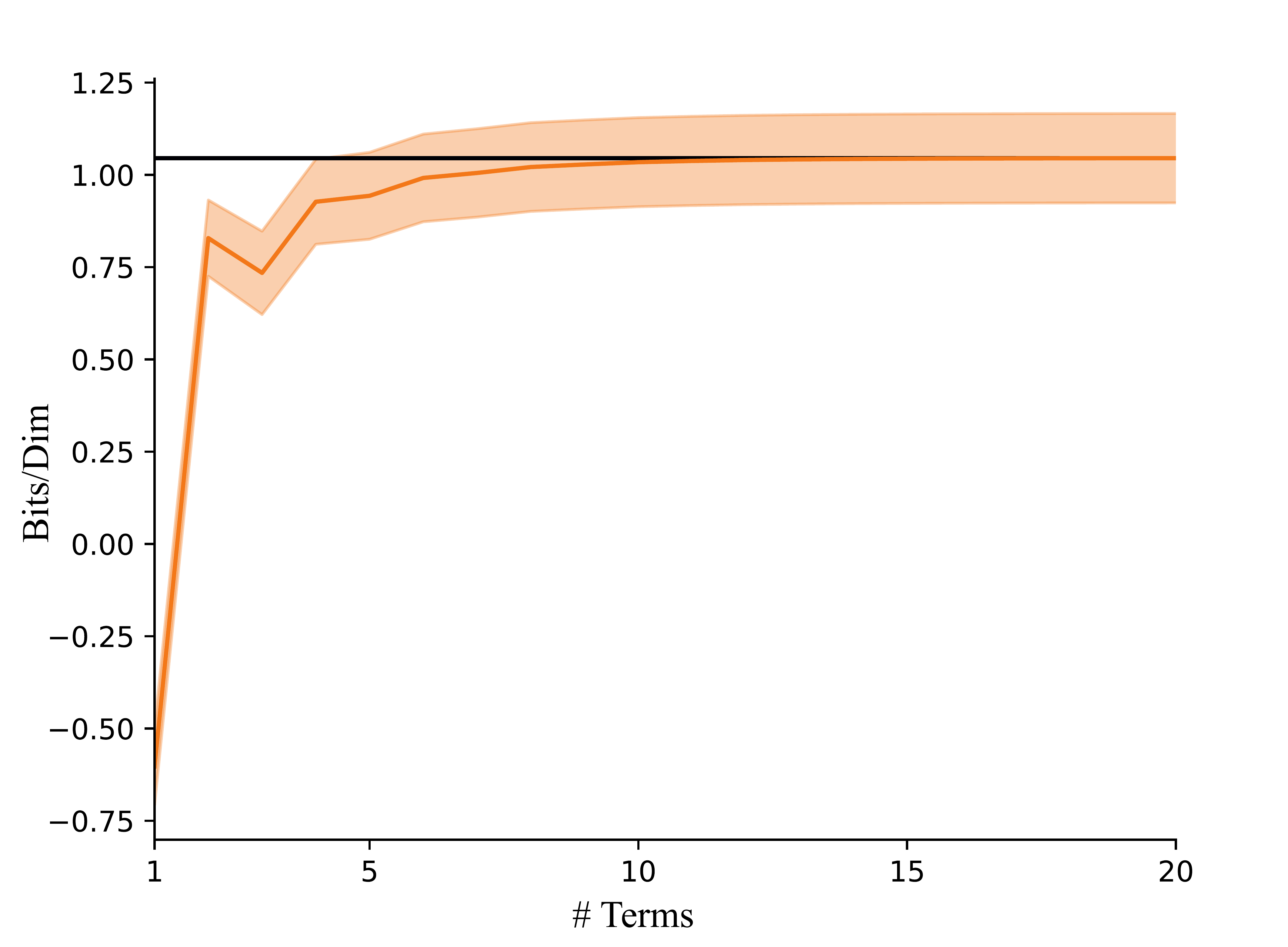} }}%
    \caption{Convergence of approximation error of log-determinant estimator when varying the number of terms used in the power series. The variance is due to the stochastic trace estimator.}%
    \label{fig:biasPlotAppendix}%
\end{figure}

\section{Additional Samples of i-ResNet flow}
\label{sec:mnist_samples}
\begin{figure}[h]
    \begin{subfigure}[b]{0.4\linewidth}
    \centering
    \includegraphics[width=\linewidth]{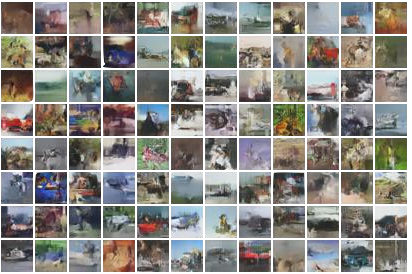}
    \caption*{CIFAR10 samples.}
    \end{subfigure}
    \begin{subfigure}[b]{0.4\linewidth}
    \centering
    \includegraphics[width=\linewidth]{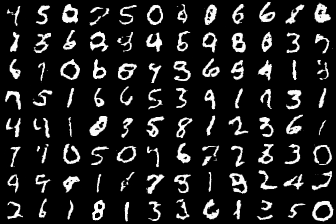}
    \caption*{MNIST samples.}
     \end{subfigure}
\label{fig:my_label}
\end{figure}

\end{document}